\documentclass[11pt]{article}

\usepackage{acl}

\usepackage{times}
\usepackage{latexsym}
\usepackage[T1]{fontenc}
\usepackage[utf8]{inputenc}
\usepackage{microtype}
\usepackage{inconsolata}
\usepackage{graphicx}
\usepackage{pifont}
\usepackage{booktabs}
\usepackage{multirow}
\usepackage{amsmath, amssymb, amsthm}
\usepackage{algorithm}
\usepackage{algpseudocode}
\usepackage[inline]{enumitem}
\usepackage{subcaption}
\usepackage{tabularx}
\usepackage{makecell}

\newcommand{\cmark}{\ding{51}}
\newcommand{\xmark}{\ding{55}}

\newtheorem{proposition}{Proposition}
\newtheorem{lemma}{Lemma}

\newcommand{\R}{\mathbb{R}}

\title{Budget-Aware Routing for Long Clinical Text}

\author{
 \textbf{Khizar Qureshi\textsuperscript{1,}},
\textbf{Geoffrey Martin\textsuperscript{2,3}},
 \textbf{Yifan Peng\textsuperscript{2,3}}
\\
 \textsuperscript{1}MIT, Cambridge, MA\\
 \textsuperscript{2}Cornell University, Ithaca, NY\\
 \textsuperscript{3}Weill Cornell Medicine, New York, NY\\
\\
\small{
  \href{mailto:kqureshi@mit.edu}{kqureshi@mit.edu},
  \href{mailto:ghm58@cornell.edu}{ghm58@cornell.edu},
  \href{mailto:yip4002@med.cornell.edu}{yip4002@med.cornell.edu}
}
}

\begin{document}
\maketitle

\begin{abstract}
A key challenge for large language models is token cost per query and overall deployment cost. Clinical inputs are long, heterogeneous, and often redundant, while downstream tasks are short and high stakes. We study budgeted context selection, where a subset of document units is chosen under a strict token budget so an off-the-shelf generator can meet fixed cost and latency constraints. We cast this as a knapsack-constrained subset selection problem with two design choices, unitization that defines document segmentation and selection that determines which units are kept.

We propose \textbf{RCD}, a monotone submodular objective that balances relevance, coverage, and diversity. We compare sentence, section, window, and cluster-based unitization, and introduce a routing heuristic that adapts to the budget regime. Experiments on MIMIC discharge notes, Cochrane abstracts, and L-Eval show that optimal strategies depend on the evaluation setting. Positional heuristics perform best at low budgets in extractive tasks, while diversity-aware methods such as MMR improve LLM generation. Selector choice matters more than unitization, with cluster-based grouping reducing performance and other schemes behaving similarly. ROUGE saturates for LLM summaries, while BERTScore better reflects quality differences. We release our code at \url{https://github.com/stone-technologies/ACL_budget_paper}.
\end{abstract}

\section{Introduction}
\label{sec:intro}

Long-context capability is now a headline feature of large language models, yet clinical deployment is constrained by simple arithmetic.
Each additional input token increases inference cost and latency, and clinical systems invoke models repeatedly across high-volume workflows.
Discharge notes, radiology reports, and evidence syntheses are long because they are templated, redundant, and stitched from multiple sources.
For example, MIMIC-IV discharge notes have a median length in the low thousands of tokens \citep{mimic_iv_2023}, and long-document benchmarks such as L-Eval routinely exceed ten thousand tokens per instance \citep{leval_2024,longbench_2023}.

Cost scales linearly with the token budget and with operational volume.
The American Hospital Association reports that community hospitals in the United States account for tens of millions of admissions annually \citep{aha_fast_facts_2025}.
Even a single large hospital generates on the order of ten thousand long notes per year \citep{Rosenbloom2010Generating}.
Let $p_\mathrm{in}$ and $p_\mathrm{out}$ denote the price per million input and output tokens, respectively.
Processing a document with input budget $B$ and output length $S$ incurs a fee.
\begin{equation}
\mathrm{Cost}(B,S) = \frac{p_\mathrm{in} B + p_\mathrm{out} S}{10^6},
\label{eq:cost}
\end{equation}
with current prices reported by API providers \citep{openai_pricing_2025}.
A reduction of even a few hundred tokens per call can translate into substantial annual savings at a hospital scale and can reduce end-to-end latency for time-sensitive decision support.

In practice, such systems support tasks ranging from automated discharge summary generation to real-time clinical decision support, where both cost and response time directly affect adoption.

These constraints motivate a system perspective where we treat context construction as a budget-constrained optimization problem that lies between raw text and an off-the-shelf generator. Classical submodular maximization under knapsack constraints provides a principled mathematical foundation \citep{nemhauser_1978,sviridenko_2004,krause_golovin_2014}, but our focus differs from the canonical setting in three ways.
First, the constraint is measured in tokens, which are coupled directly to monetary and latency budgets.
Second, the objective must behave predictably across budgets, since practitioners often operate under heterogeneous constraints across users, devices, and clinical services.
Third, the system must choose among multiple selectors and representations, rather than commit to a single surrogate objective.
This motivates the routing layer studied in this paper.

Our central finding is that the optimal selection strategy depends on the evaluation paradigm. For extractive evaluation, positional heuristics such as Lead dominate at low budgets because clinical documents front-load important content and ROUGE rewards lexical overlap. For LLM-based generation, diversity-aware selection (MMR) consistently outperforms positional methods because the LLM benefits from non-redundant input rather than verbatim coverage. A budget-aware routing heuristic captures these regime-dependent differences at near-oracle performance.

\subsection{Our contributions}
\label{sec:contrib}

We adopt a system-first view in which the context builder is treated as an explicit module with a measurable input budget, compute footprint, and an auditable output context.
The module receives a long document $D$ and a budget $B$, and emits a shorter context $C_B$ that is passed to a downstream generator.
This decoupling facilitates model-agnostic analysis of unitization and selection and permits comparisons across budget regimes.

Our contributions are as follows.
\begin{enumerate*}[label=(\roman*)]
\item We formalize budgeted context construction with explicit \textbf{unitization} and \textbf{selection} stages, and we evaluate sentence-, window-, section-, and cluster-based unitization strategies.
\item We introduce \textbf{RCD} (Relevance-Coverage-Diversity), a monotone submodular objective that combines relevance, facility-location coverage \citep{lin_bilmes_2011}, and log-determinant diversity \citep{kulesza_taskar_2012} under a token knapsack constraint.
\item We implement a unified selector suite including lead baselines \citep{see2017get}, shuffled controls, sliding-window selection, hierarchical expansion, graph-based semantic clustering, maximal marginal relevance (MMR) \citep{carbonell_goldstein_1998}, and RCD.
\item We propose a lightweight \textbf{budget-aware router} that selects an algorithm based on budget and document statistics, and we evaluate it against oracle upper and lower bounds.
\item We report an experimental study across MIMIC discharge notes, Cochrane abstracts, and L-Eval long-document summarization tasks, including both extractive evaluation and end-to-end LLM generation.
\end{enumerate*}

\section{Related Work}
\label{sec:related}

A recent wave of benchmarks evaluates model behavior as input length grows.
L-Eval collects long documents with human-written questions and summarization tasks \citep{leval_2024}.
LongBench evaluates long context understanding across multiple tasks \citep{longbench_2023}.
RULER and needle-in-a-haystack evaluations isolate retrieval behavior \citep{ruler_2024, needle_haystack_2023}.
Our work is orthogonal to these efforts: we change the context that reaches the model rather than the model itself. A pragmatic response to a long context is to compress prompts or select salient content.
LLMLingua compresses prompts while preserving answer quality \citep{llmlingua_2023}.
Retrieval-augmented generation selects passages from external corpora \citep{rag_2020}.
Our setting differs in that we select from a single long document under a strict token budget rather than retrieving from a large corpus. Submodular functions formalize diminishing returns and have a long history in summarization \citep{lin_bilmes_2011}.
Maximal marginal relevance trades relevance against redundancy using a similarity kernel \citep{carbonell_goldstein_1998}.
Determinantal point processes provide diversity through log-determinant objectives \citep{kulesza_taskar_2012}. We build directly on these ideas, adding explicit token budgets and a routing layer that adapts to the budget regime.

Neural extractive summarizers such as BertSumExt \citep{liu_lapata_2019} learn sentence-level selection end-to-end. Our work is complementary as we study lightweight, training-free selectors under explicit token budgets, isolating the effect of selection policy from learned extraction.

\section{Methodology}

\subsection{Framework overview}
Given a long document $D$ and a token budget $B$, our pipeline has four stages.
\begin{enumerate*}[label=(\arabic*)]
\item \textbf{Unitization}: We segment $D$ into sentence-level candidates and estimate costs $c_i$.
\item \textbf{Representation}: We encode each unit with lexical or semantic features and compute relevance signals.
\item  \textbf{Routing}: We map $(D,B)$ to a selector regime, using a lightweight heuristic that trades off expected marginal utility against computational cost.
\item  \textbf{Selection}: We choose a budget-feasible subset with one of several algorithms, including lead, shuffled, greedy MMR, hierarchical selection, graph-based semantic clustering, and our RCD objective.
\end{enumerate*}
Figure~\ref{fig:architecture} summarizes the end-to-end workflow and the associated evaluation loop on MIMIC, Cochrane, and L-Eval.

\begin{figure*}[t]
\centering
\includegraphics[width=0.9\textwidth]{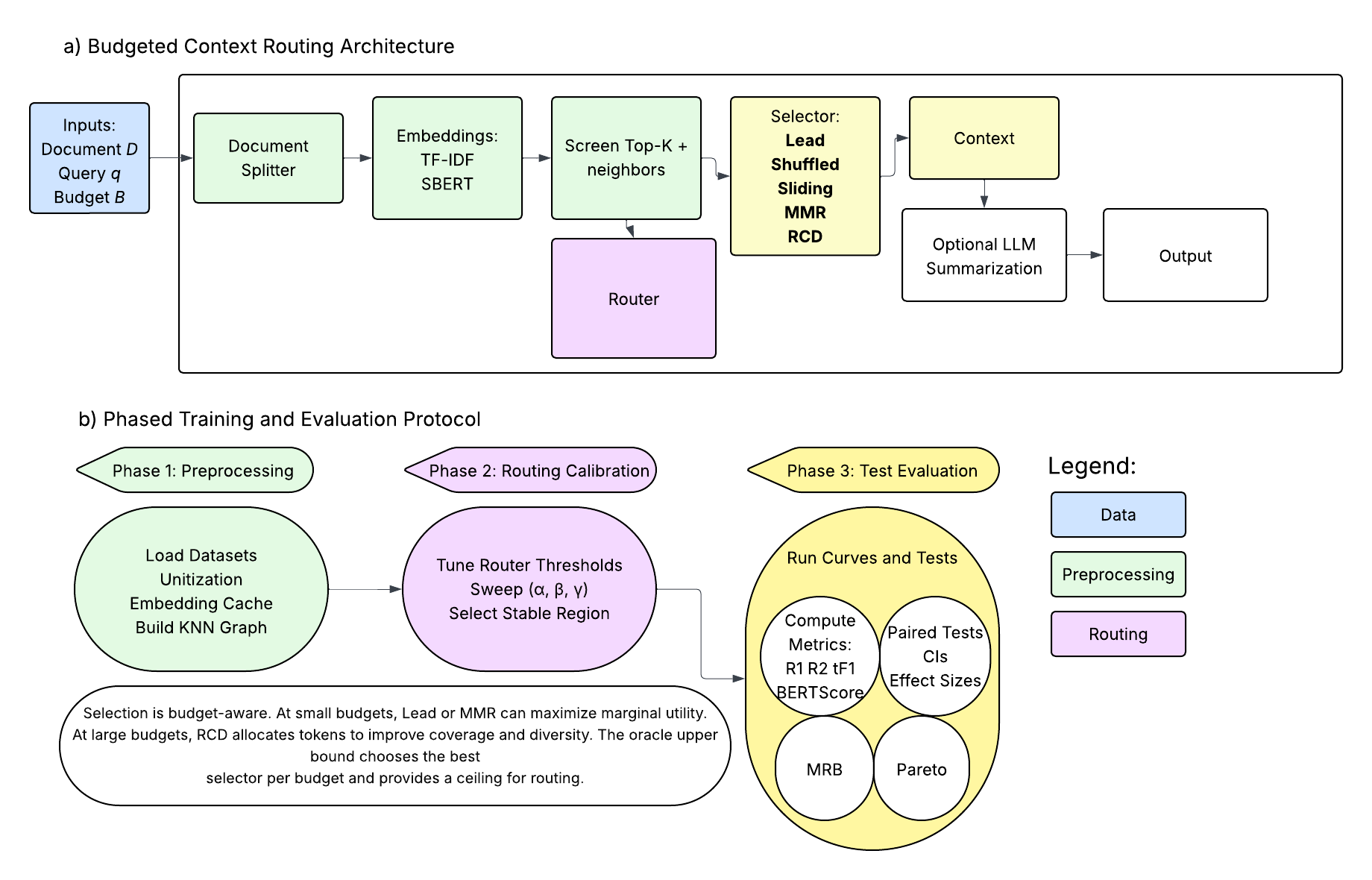}
\caption{End-to-end architecture of the budgeted context construction framework. We begin with sentence-level unitization and feature extraction, route each (document, budget) pair to a selector regime, construct a budget-feasible context, and evaluate with ROUGE and token-level F1. The bottom phase bar reflects the experimental protocol: preprocessing, routing calibration, and held-out evaluation across MIMIC, Cochrane, and L-Eval.}
\label{fig:architecture}
\end{figure*}

\subsection{Problem Formulation}
\label{sec:problem}

Let a document $D$ be partitioned into units $u_1,\dots,u_n$ by a unitization function, where each unit $u_i$ incurs a token cost $c_i$.
Given a budget $B$, the goal is to select a subset $S \subseteq \{1,\dots,n\}$ such that
$
\sum_{i \in S} c_i \le B
$
while maximizing downstream utility. We assume access to embeddings $\phi(u_i) \in \R^d$ and define a similarity kernel $k(i,j) = \langle \phi(u_i), \phi(u_j) \rangle$.
A selection policy maps $(u_{1:n}, B)$ to a subset $S$.
The selected units are concatenated in document order and either evaluated directly in an extractive setting or passed to a large language model (LLM) for abstractive summarization.
Algorithm~\ref{alg:pipeline} summarizes the complete system.

\begin{table*}[t]
\centering
\small
\setlength{\tabcolsep}{3pt}
\begin{tabularx}{\linewidth}{l@{}cccc>{\raggedright\arraybackslash}X>{\raggedright\arraybackslash}X}
\toprule
Study & \makecell[c]{Long-context\\eval} & \makecell[c]{Budgeted\\context}  & \makecell[c]{Budget-aware\\routing} & \makecell[c]{Multi-\\domain} & Selector family & Objective\\
\midrule
L-Eval \citep{leval_2024} & \cmark & \xmark  & \xmark & \cmark & -- & --  \\
LongBench \citep{longbench_2023} & \cmark & \xmark & \xmark & \cmark & -- & --  \\
RULER \citep{ruler_2024} & \cmark & \xmark & \xmark & \xmark & -- & --  \\
MMR selection \shortcite{carbonell_goldstein_1998} & \xmark & \cmark & \xmark & \xmark & MMR & submodular variants  \\
Facility \citep{lin_bilmes_2011} & \xmark & \cmark & \xmark & \xmark & submodular & facility location \\
\midrule
\textbf{This work} & \cmark & \cmark  & \cmark & \cmark & lead, MMR, hier, GSC, RCD & knapsack + submodular \\
\bottomrule
\end{tabularx}
\caption{Positioning relative to prior work. Existing long-context benchmarks evaluate models at fixed or implicit context lengths, while classical selectors optimize within a budget but lack benchmark-driven, budget-aware routing across methods. We unify these threads by treating context construction as a budgeted operations research problem and evaluating the same selector family across clinical, scientific, and benchmark corpora.}
\label{tab:related_comparison}
\end{table*}



\begin{algorithm}[h!]
\small
\caption{Budgeted context selection}
\label{alg:pipeline}
\begin{algorithmic}[1]
\Require Document $D$, budget $B$, unitization $\mathcal{U}$, selector $m$
\Ensure Selected context $y$
\State Build units $u_{1:n} \gets \mathcal{U}(D)$
\State Compute costs $c_i$ and embeddings $\phi(u_i)$
\State Compute relevance $r_i$ and similarity kernel $k$
\State Select subset $S \gets m(u_{1:n}, c_{1:n}, r_{1:n}, K, B)$
\State Concatenate context $y \gets \text{concat}(\{u_i : i \in S\})$
\State \Return $y$
\end{algorithmic}
\end{algorithm}

\subsection{Unitization Strategies}
\label{sec:unitization}

Unitization determines the granularity and structure of selectable units.
We explore four strategies that create different unit types before selection.

\textbf{Sentence-Unit.}
Splits at sentence boundaries, treating each sentence as an independent unit. This provides maximum flexibility but may fragment coherent content.

\textbf{Section-Unit.}
Parses section headers (e.g., \texttt{CHIEF COMPLAINT} or \texttt{HOSPITAL COURSE}) and assigns sentences to their corresponding sections, enabling section-aware scoring.

\textbf{Window-Unit.}
Creates overlapping chunks (base 50 words) with size varying by local content density.

\textbf{Cluster-Unit.}
Groups semantically related sentences using a similarity graph with cosine similarity and proximity decay; connected components become units.

\subsection{Selection Algorithms}
\label{sec:selectors}

Given units from any unitization strategy, selection determines which units to keep under the budget.

\paragraph{Baselines}

The simplest baseline is \textbf{Lead Selection}, which takes units from the beginning of the document until the budget is exhausted.
It requires minimal computation beyond measuring unit lengths and serves as a strong baseline when editorial conventions prioritize important content.
\textbf{Shuffled selection} randomizes unit order before applying lead selection. This strategy allows us to quantify the contribution of positional information by comparing its performance with that of standard lead selection.
\textbf{Sliding selection} finds the contiguous sequence of units with the highest total relevance under the budget.
It preserves local coherence and performs well when salient information is clustered in the document.
\textbf{Hierarchical selection} identifies high-relevance anchor sentences and progressively adds adjacent context until the budget is filled.
This strategy balances readability with relevance.
\textbf{GraphCluster selection} builds a sentence similarity graph and traverses high-relevance connected components, selecting representative sentences from each component within the budget.
This strategy encourages topical coverage by allocating budget across semantic clusters.

\paragraph{Maximal Marginal Relevance (MMR)}
\label{sec:mmr}

MMR selects units iteratively, balancing relevance against redundancy \citep{carbonell_goldstein_1998}.
At each step, MMR scores each candidate unit $i$ as
\begin{equation}
\mathrm{MMR}(i \mid S_t) = \lambda r_i - (1-\lambda)\max_{j \in S_t} k(i,j),
\end{equation}
where $r_i$ measures relevance to the query or document centroid, and $\lambda \in [0,1]$ controls the relevance–redundancy tradeoff.
The redundancy term is updated at each iteration based on the currently selected set $S_t$.
This enables adaptive diversification: once a topic is covered, additional units on the same topic are penalized.

\paragraph{Relevance, Coverage, Diversity (RCD)}
\label{sec:rcd}

RCD defines an objective that explicitly separates three criteria.
The relevance term $R(S) = \sum_{i \in S} r_i$ encourages selecting units aligned with the query.
Facility location coverage $C(S) = \sum_{i=1}^n \max_{j \in S} k(i,j)$ encourages selecting a set that represents the entire document.
Log-determinant diversity $D(S) = \log \det(I + \eta K_S)$ where $K_S$ is the kernel submatrix indexed by $S$, provides a principled diversity model.
The combined objective is:
\begin{equation}
F(S) = \alpha R(S) + \beta C(S) + \gamma D(S).
\end{equation}
If $k(i,j) \ge 0$ and $K$ is positive semidefinite, then $F$ is monotone submodular. This property enables approximation guarantees under knapsack constraints via lazy greedy selection.



\subsection{Routing by Budget Regime}
\label{sec:routing}

The optimal selection strategy depends on the budget and evaluation paradigm.
At small budgets, positional heuristics perform well because important content is front-loaded, and extractive metrics reward lexical overlap.
At moderate budgets, redundancy becomes limiting, making diversity-aware selection advantageous.
At large budgets, coverage dominates, as the challenge shifts to representing diverse topics. We formalize this intuition as a \textbf{budget-aware routing policy} that maps budget $B$ to a selection method.
Two document statistics inform the routing decision.
The \textbf{front-loading index} measures the fraction of total relevance captured by units fitting in the budget:
\begin{equation}
\phi = \frac{\sum_{i=1}^{m} r_i}{\sum_{i=1}^n r_i},
\end{equation}
where $m$ is the number of units fitting in $B$.
The \textbf{redundancy index} measures local repetition:
\begin{equation}
\rho = \frac{1}{n-1}\sum_{i=1}^{n-1} k(i,i+1).
\end{equation}

A simple heuristic then chooses the algorithm based on budget:
\begin{equation}
\pi(B)=
\begin{cases}
\textsc{Lead} & \text{if } B \le B_1, \\
\textsc{MMR} & \text{if } B_1 < B \le B_2, \\
\textsc{RCD} & \text{otherwise.}
\end{cases}
\end{equation}
Thresholds $B_1$ and $B_2$ are tuned on the validation set to maximize the evaluation metric (here, mean ROUGE-1) across the budget sweep.

\section{Experimental Setup}
\label{sec:experiments}

We evaluate on three datasets spanning clinical writing, scientific abstracts, and long document benchmarks (Table~\ref{tab:data_stats}).
\textbf{MIMIC} contains discharge notes paired with Brief Hospital Course summaries written by physicians \citep{mimic_iv_2023}.
\textbf{Cochrane} provides abstract-to-conclusion pairs from systematic reviews, representing clinical evidence synthesis.
\textbf{L-Eval} includes summarization tasks across government reports, meetings, news, and patents \citep{leval_2024}.

\begin{table}[t]
\centering
\small
\begin{tabular}{lrrrr}
\toprule
Dataset & \multicolumn{2}{c}{Input tokens} & \multicolumn{2}{c}{Target tokens} \\
\cmidrule(r){2-3}\cmidrule(l){4-5}
& median & mean & median & mean \\
\midrule
MIMIC & 2175 & 2272 & 456 & 567 \\
Cochrane & 542 & 600 & 149 & 166 \\
L-Eval & 13069 & 16027 & 324 & 310 \\
\bottomrule
\end{tabular}
\caption{Dataset scale and length statistics. Token counts are either provided by the dataset or estimated from word counts using an empirical tokens-per-word ratio.}
\label{tab:data_stats}
\end{table}

We evaluate across token budgets $B \in \{256, 512, 1024, 2048, 4096, 8192, 16384\}$.
Selection methods use sentence-level units for extractive evaluation, while unitization experiments consider Sentence-, Section-, Window-, and Cluster-Unit crossed with Lead, MMR, and our proposed RCD method.

For end-to-end evaluation, the selected context is passed to GPT-4o (using institutional-supported Azure OpenAI service and opting out of human review of the data) for abstractive summary generation. The extractive setting evaluates the context builder as an independent module: ROUGE between the selected context and the reference acts as a coverage proxy, measuring how much reference-aligned content survives the selection step. This separation isolates selector quality from generator quality before committing to repeated LLM calls.

We report ROUGE-1 and ROUGE-2 F1 scores, which measure unigram and bigram overlap between generated text and reference summary. For end-to-end LLM evaluation, we also report BERTScore F1, which measures semantic similarity using contextualized embeddings from DeBERTa-xlarge-MNLI with baseline rescaling~\citep{bertscore_2020}. Scores near zero indicate chance-level similarity, and scores of 0.10--0.15 indicate moderate semantic alignment typical of abstractive summaries. Full metric definitions are provided in Appendix~\ref{app:metrics}.

\section{Results}
\label{sec:results}

\subsection{Extractive Selection Results}

Table~\ref{tab:r1_main} reports ROUGE-1 scores for sentence-level selection methods across datasets and budgets.

\begin{figure*}[t]
\centering
\begin{subfigure}{0.32\linewidth}
\centering
\includegraphics[width=\linewidth]{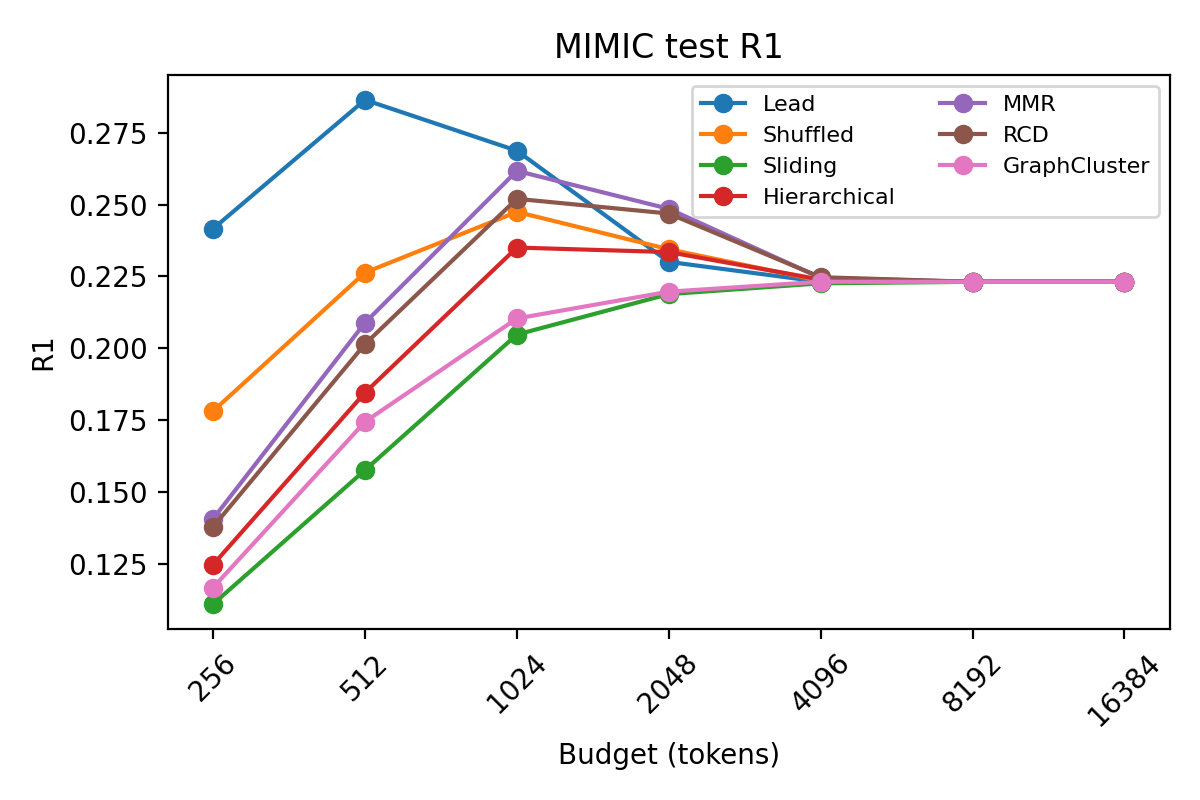}
\caption{MIMIC}
\end{subfigure}
\begin{subfigure}{0.32\linewidth}
\centering
\includegraphics[width=\linewidth]{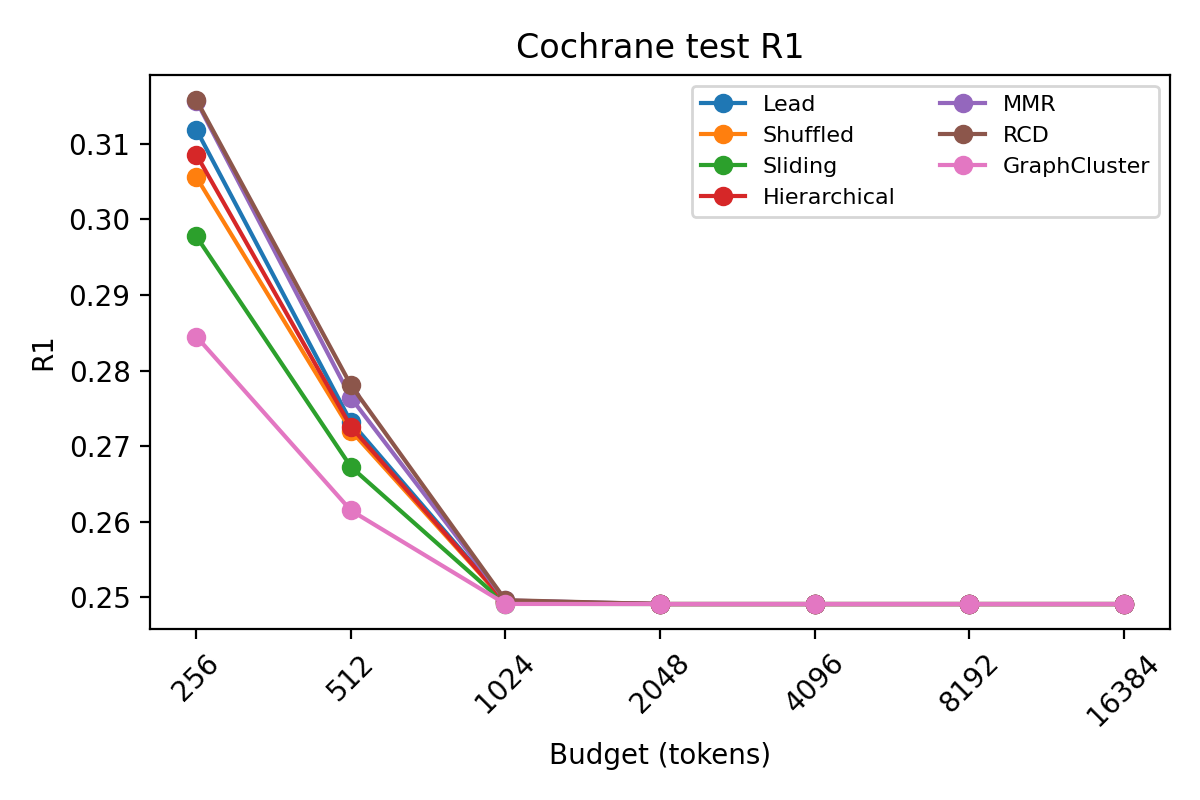}
\caption{Cochrane}
\end{subfigure}
\begin{subfigure}{0.32\linewidth}
\centering
\includegraphics[width=\linewidth]{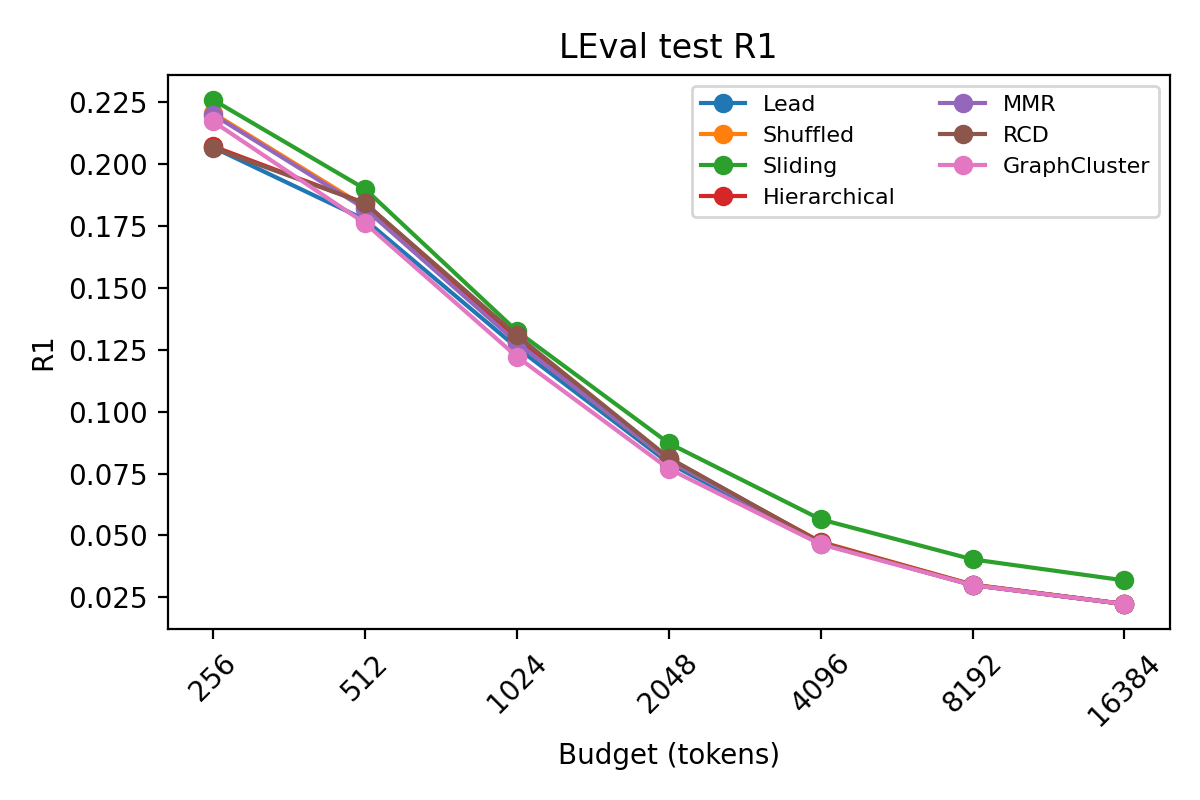}
\caption{L-Eval}
\end{subfigure}
\caption{ROUGE-1 F1 as a function of budget. The strongest method depends on the budget regime.}
\label{fig:curves_r1}
\end{figure*}

\begin{figure*}[t]
\centering
\begin{subfigure}{0.32\linewidth}
\centering
\includegraphics[width=\linewidth]{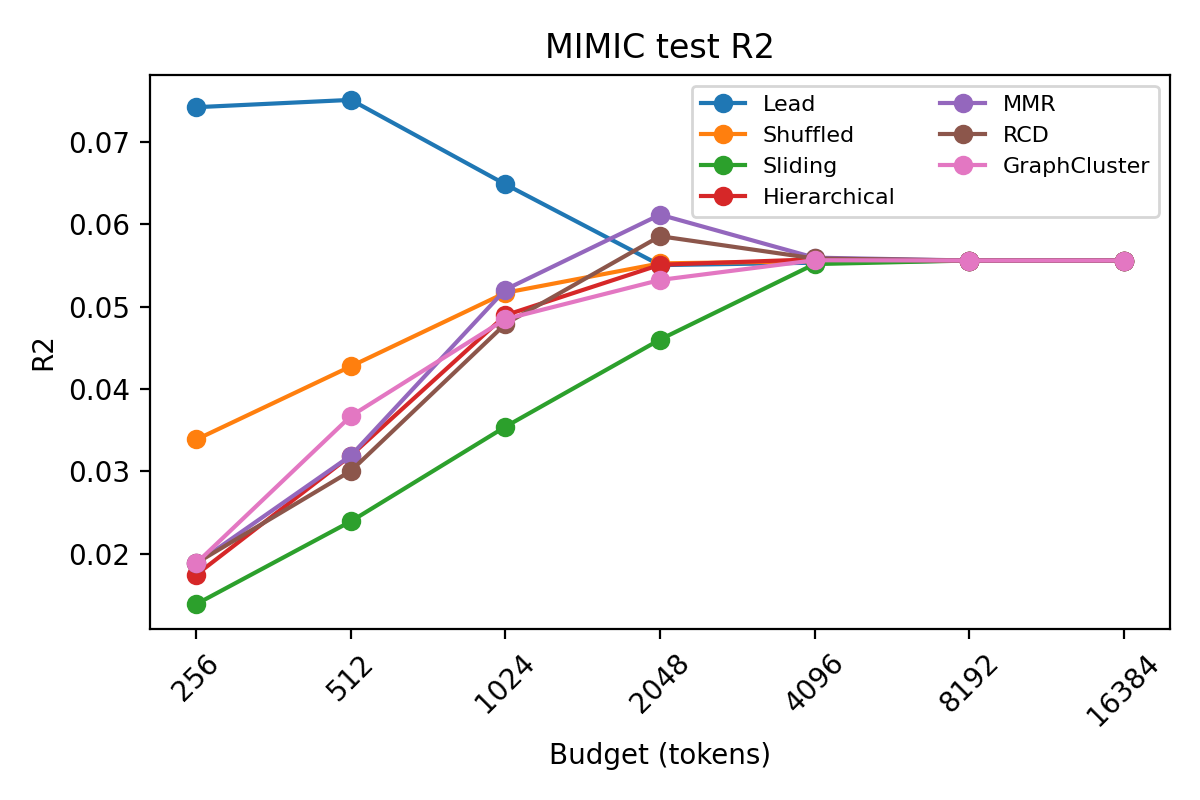}
\caption{MIMIC}
\end{subfigure}
\begin{subfigure}{0.32\linewidth}
\centering
\includegraphics[width=\linewidth]{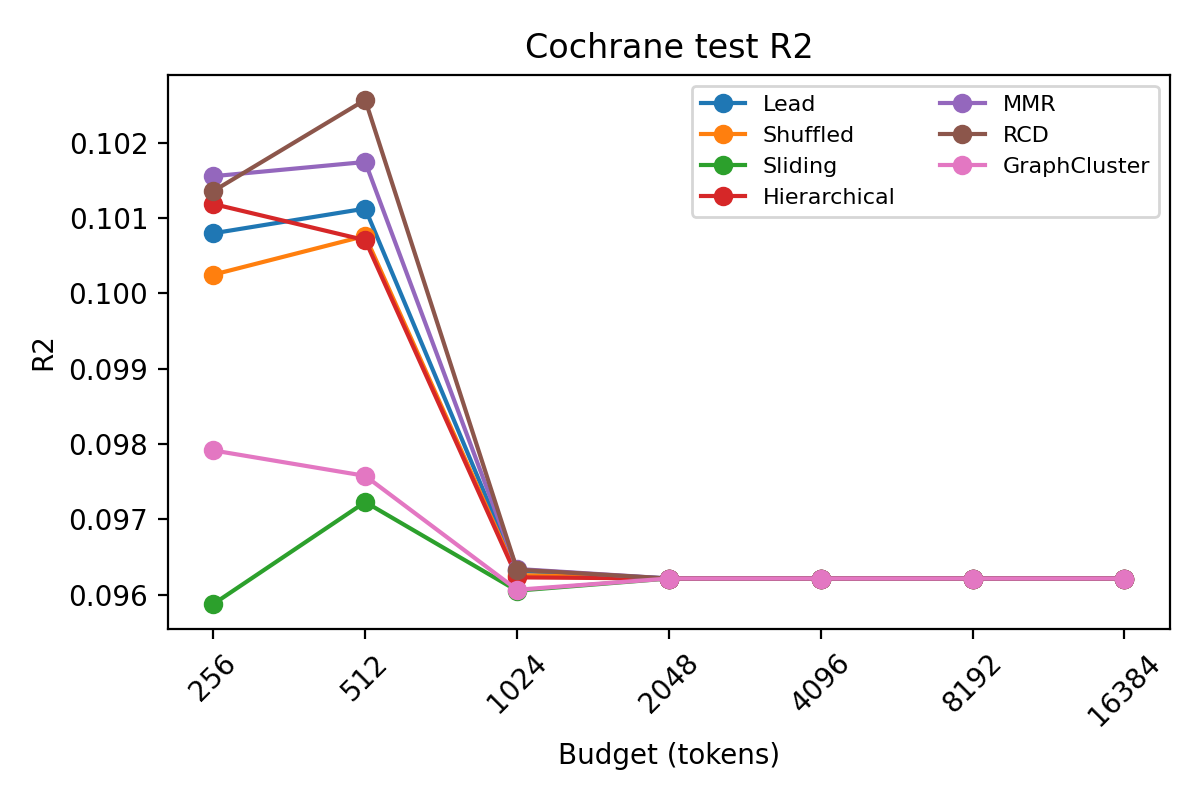}
\caption{Cochrane}
\end{subfigure}
\begin{subfigure}{0.32\linewidth}
\centering
\includegraphics[width=\linewidth]{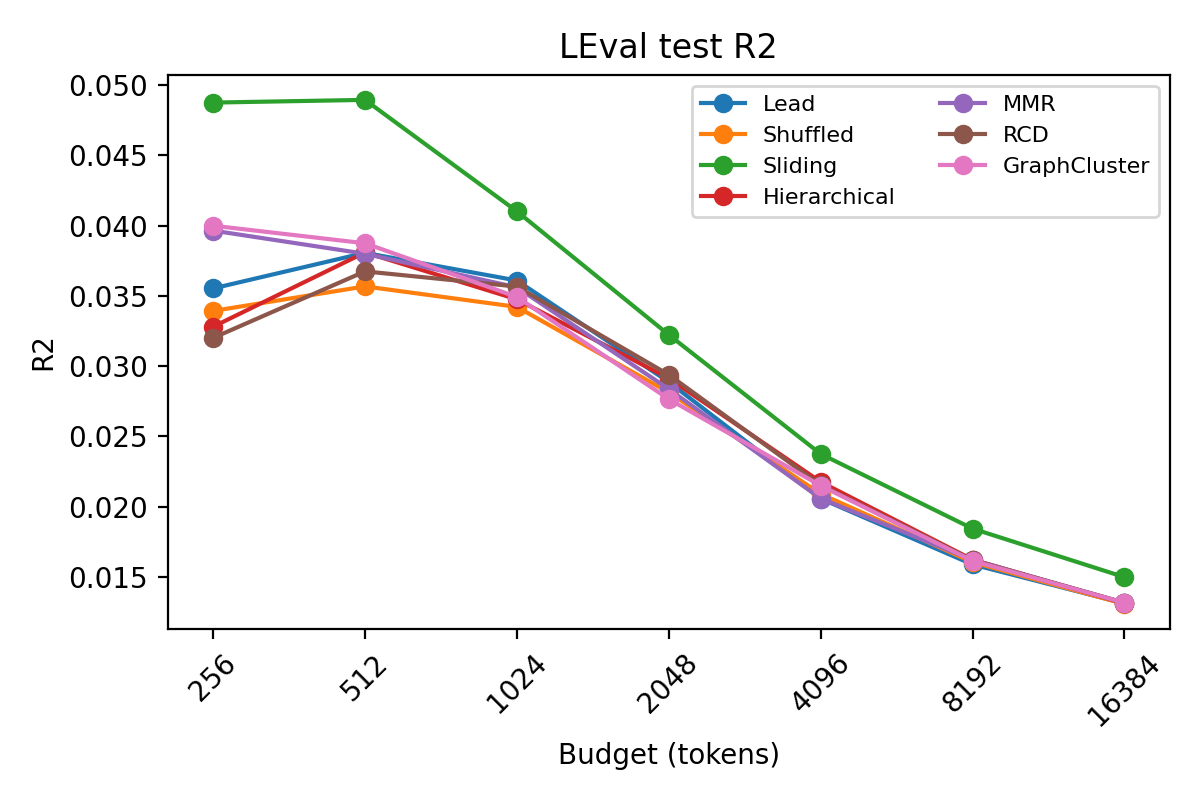}
\caption{L-Eval}
\end{subfigure}
\caption{ROUGE-2 F1 as a function of budget. Bigram overlap highlights redundancy effects at larger budgets.}
\label{fig:curves_r2}
\end{figure*}


On MIMIC, Lead performs competitively at 256--512 tokens due to front-loaded content in the discharge notes.
With larger budgets, diversity-aware selectors such as MMR and RCD help reduce redundancy, though ROUGE-1 gains are modest.
At very large budgets, all methods converge as the entire document fits within the budget. Cochrane abstracts are short, leading to selection methods converging quickly.
At 256--512 tokens, RCD typically performs best, showing that it prioritizes broader coverage over front-loading.

For L-Eval, where key content is span-specific, the Sliding selector is particularly effective for medium- and large-budget projects.
At the smallest budget, non-redundant selection methods remain competitive.

Figures~\ref{fig:curves_r1} and~\ref{fig:curves_r2} show ROUGE-1 and ROUGE-2 F1 across budgets. On MIMIC, Lead is strongest at low budgets due to front-loaded content, but MMR dominates at moderate budgets (1024) tokens by filtering redundancy. On Cochrane, short documents cause rapid convergence across methods. On L-Eval, Sliding selection consistently outperforms other methods by identifying concentrated spans of relevant content, while positional heuristics (Lead) provide little advantage over random selection. ROUGE-2 shows similar patterns with greater method separation.

\subsection{Pareto View and Budget Efficiency}

Figure~\ref{fig:pareto} shows the best achievable ROUGE-1 at each budget over evaluated policies.
This upper envelope defines a Pareto frontier between cost and quality, and represents the natural target for a routing policy.

\begin{figure}[t]
\centering
\includegraphics[width=0.85\linewidth]{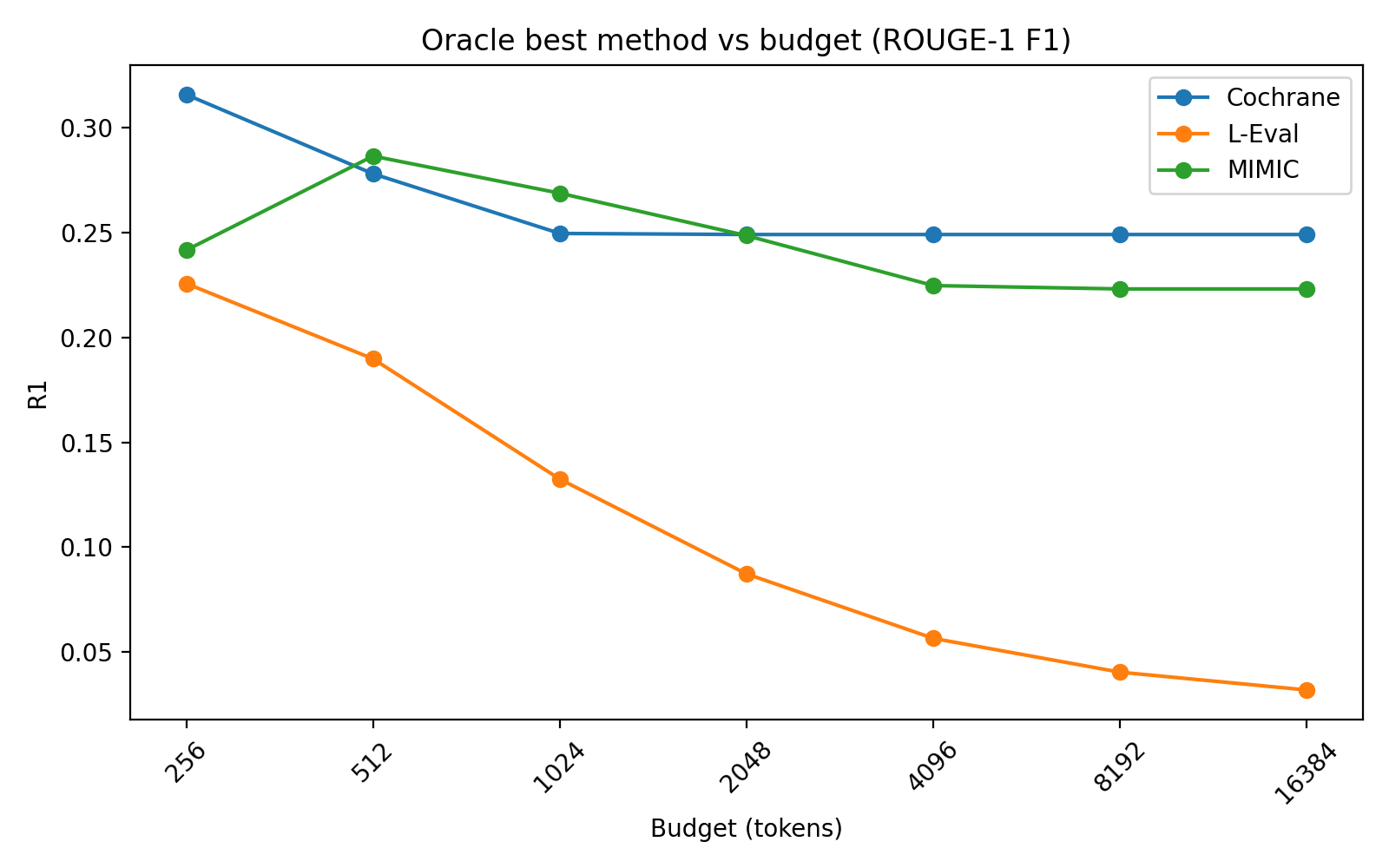}
\caption{Pareto frontier for ROUGE-1 over budgets for each dataset, defined as the best score over evaluated policies at each budget.}
\label{fig:pareto}
\end{figure}

\subsection{End-to-End Results with LLM Generation}
\label{sec: end-to-end}

Tables~\ref{tab:mimic_llm_rouge} and~\ref{tab:mimic_llm_bert} report ROUGE-1 and BERTScore F1 on MIMIC when the selected context is passed to GPT-4o for abstractive summary generation. We evaluate all combinations of unitization strategies and selection algorithms.

To contextualize absolute score levels, we evaluate a full-context baseline that passes the entire MIMIC document to GPT-4o with no budget constraint (Table~\ref{tab:fullcontext}).

\begin{table}[t]
\centering
\small
\setlength{\tabcolsep}{4pt}
\begin{tabular}{lccc}
\toprule
& R-1 & R-2 & BS \\
\midrule
Full context & .277 & .071 & .106 \\
Sent./MMR, $B{=}512$ & .277 & -- & .117 \\
Sent./MMR, $B{=}1024$ & .282 & -- & .120 \\
\bottomrule
\end{tabular}
\caption{Full-context baseline on MIMIC with GPT-4o versus budgeted selection ($N{=}100$). R-1: ROUGE-1, R-2: ROUGE-2, BS: BERTScore. Sentence/MMR at 512 tokens matches full-context ROUGE-1 and exceeds it at 1024.}
\label{tab:fullcontext}
\end{table}

Budget selection at 512 tokens, roughly a quarter of median document length, recovers full-context performance, and at 1024 tokens exceeds it. This is consistent with redundancy in clinical text since removing repeated content can help the LLM focus on relevant information.

MMR achieves the highest ROUGE-1 at low budgets across most unitization strategies. At 256 tokens, Hierarchical+MMR achieves 0.267, outperforming Lead by 2.1 points. While the margin is modest, it is consistent: MMR's redundancy penalty improves context even when the LLM generates abstractive summaries.

BERTScore highlights differences that ROUGE misses. While ROUGE-1 varies by only 1--3 points across methods, BERTScore shows clearer separation. MMR consistently outperforms Lead (0.111 vs 0.091 at 256 for Hierarchical), suggesting that non-redundant context helps the LLM produce more semantically appropriate summaries.

Hierarchical-Unit and Sentence-Unit produce nearly identical results across all selectors and budgets, confirming that current selectors do not exploit section metadata. Cluster-Unit underperforms at low budgets (0.230 vs 0.247 for Lead at 256 tokens), though the gap narrows at higher budgets.

Finally, ROUGE scores plateau around 0.27--0.28 regardless of budget or method. This ceiling likely reflects the abstractive nature of LLM generation: the model paraphrases rather than copying, limiting lexical overlap with the reference.

\subsection{Position Dependence}

\begin{table}[t]
\centering
\small
\begin{tabular}{lcccc}
\toprule
& 256 & 512 & 1024 & 2048 \\
\midrule
Lead & 0.247 & 0.264 & 0.269 & 0.280 \\
Shuffled & 0.230 & 0.257 & 0.268 & 0.271 \\
\midrule
$\Delta$ & +0.017 & +0.007 & +0.001 & +0.009 \\
\bottomrule
\end{tabular}
\caption{Position dependence: Lead vs.\ Shuffled on MIMIC with LLM generation (N=100).}
\label{tab:position}
\end{table}

Positional signals are present but modest (Table \ref{tab:position}). At 256 tokens, Lead outperforms Shuffled by 1.7 points (0.247 vs 0.230). This suggests that while MIMIC discharge notes front-load relevant information, the advantage is limited when an LLM processes the context. The model can identify salient information regardless of position. By 512 tokens and beyond, the performance gap is negligible.  

\subsection{ROUGE vs.\ BERTScore}
ROUGE measures lexical overlap between system output and a reference summary \citep{lin_2004}.
We report ROUGE-1 and ROUGE-2 as $n$-gram F1 scores, which are appropriate when the desired summary copies domain-specific phrases and entities.
However, when an LLM is used as the generator, outputs often paraphrase rather than copy, compressing the dynamic range of ROUGE across methods.

To complement ROUGE, we report BERTScore F1 \citep{bertscore_2020}.
BERTScore aligns tokens via contextual embeddings and computes soft precision and recall, which makes it sensitive to semantic similarity under paraphrase.
We use baseline rescaling, which subtracts the expected similarity of unrelated sentence pairs, so values near zero indicate chance-level alignment and values around $0.10$ indicate moderate semantic agreement for abstractive summaries.

On MIMIC with LLM generation, ROUGE-1 concentrates in a narrow band (roughly $0.23$ to $0.28$), while BERTScore separates selectors more clearly.
At $B=256$ tokens, MMR achieves the highest BERTScore across unitizations, consistent with the intuition that redundancy control improves the semantic content available to the generator even when lexical overlap changes little.
Together, these metrics distinguish two failure modes.
A selector can achieve a high ROUGE by copying frequent boilerplate.
A selector can achieve a high BERTScore by covering clinically salient concepts that the generator paraphrases.
We therefore report both throughout.

\subsection{Routing Results}

Figure~\ref{fig:routing_bounds} compares the performance of the optimized routing heuristic against the oracle upper and lower bounds. The heuristic closely tracks the upper bound across all datasets, and the gap between the best and worst methods is narrow (1--2 points), suggesting that budget matters more than selector choice. The degradation on MIMIC beyond 1024 tokens likely reflects signal dilution from including less relevant content.

\begin{figure*}[t]
\centering
\begin{subfigure}{0.32\linewidth}
\centering
\includegraphics[width=\linewidth]{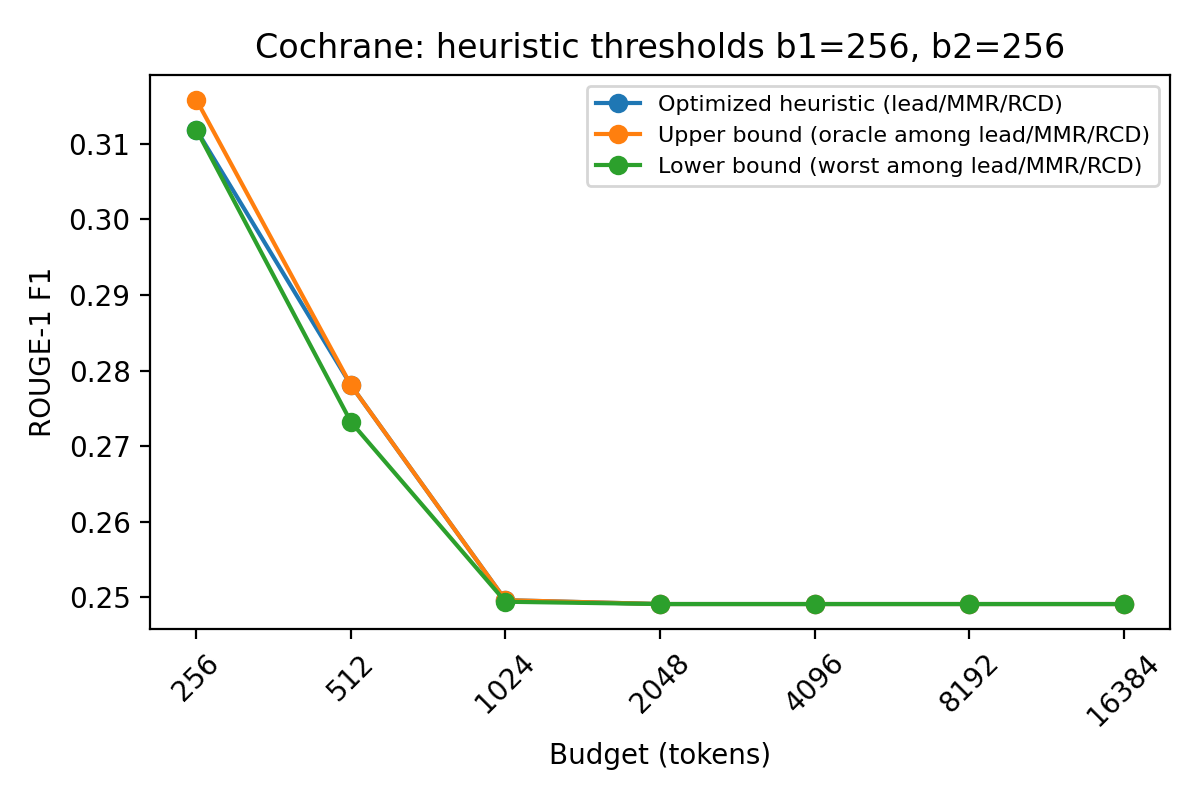}
\caption{Cochrane}
\end{subfigure}
\begin{subfigure}{0.32\linewidth}
\centering
\includegraphics[width=\linewidth]{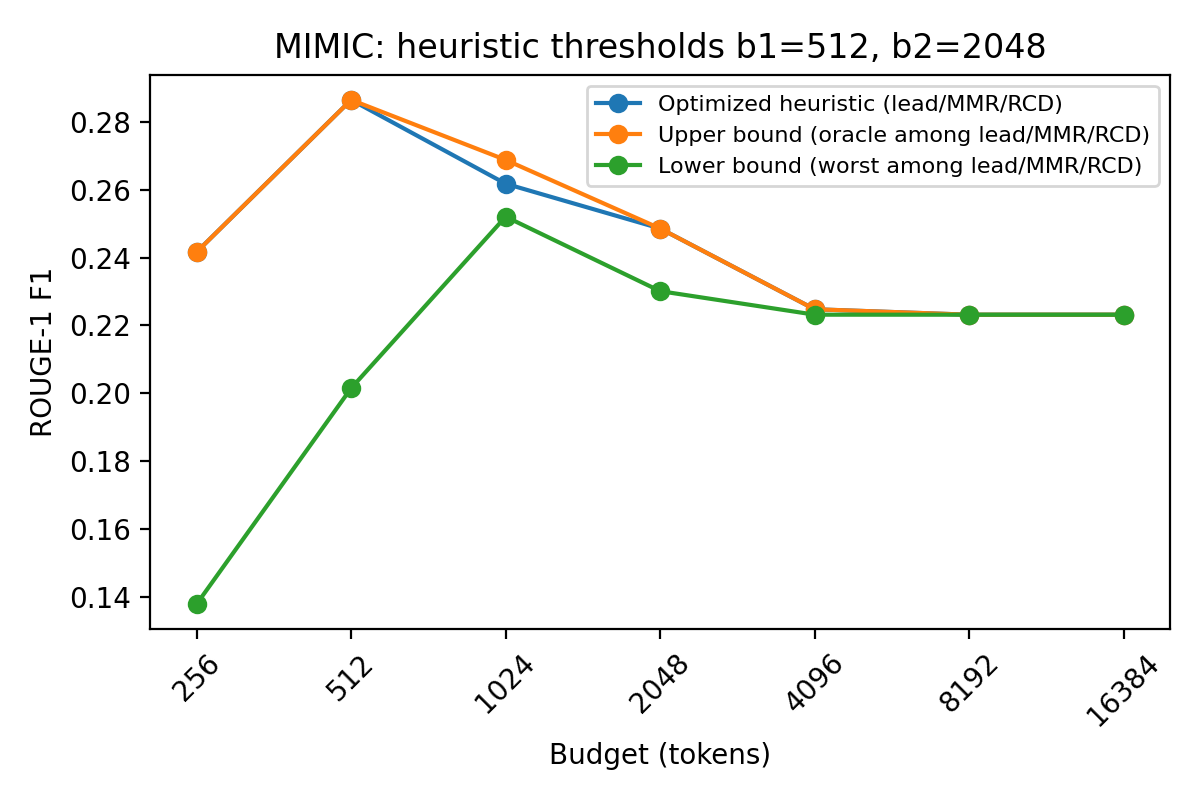}
\caption{MIMIC}
\end{subfigure}
\begin{subfigure}{0.32\linewidth}
\centering
\includegraphics[width=\linewidth]{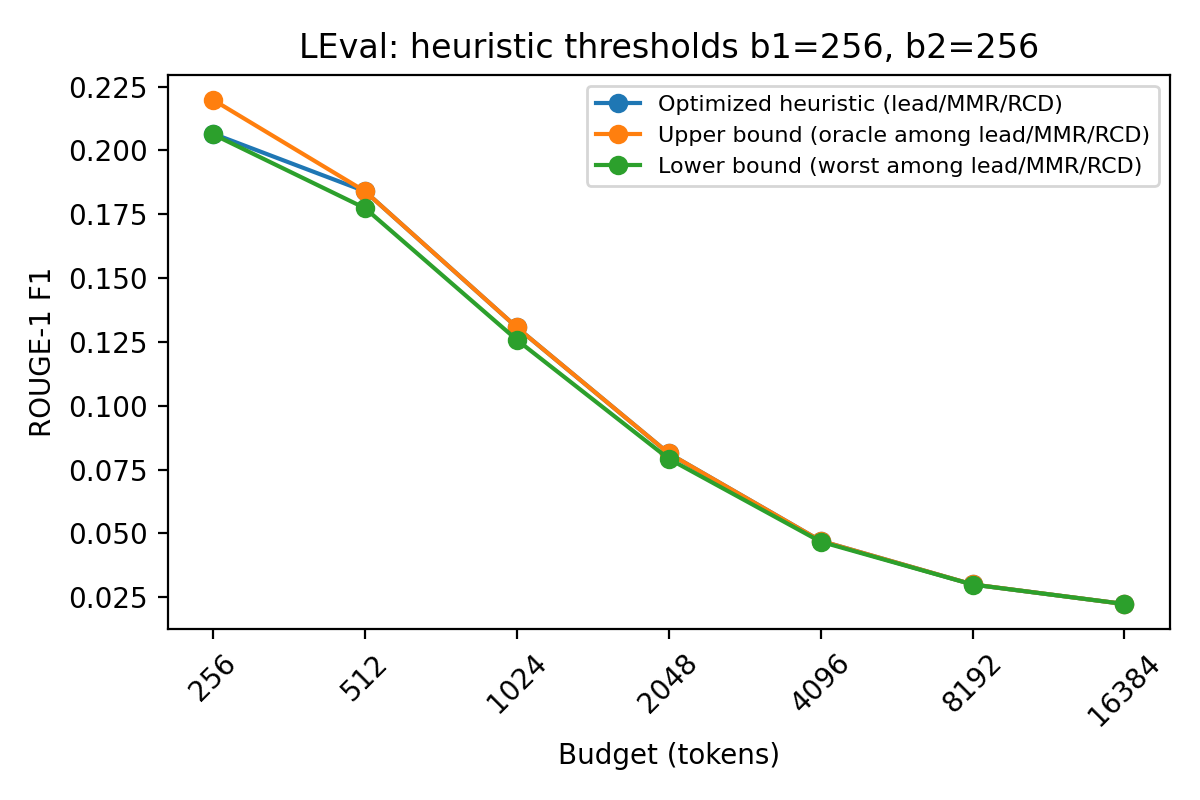}
\caption{L-Eval}
\end{subfigure}
\caption{Optimized routing heuristic performance on test, with oracle upper and lower bounds over Lead, MMR, and RCD.}
\label{fig:routing_bounds}
\end{figure*}

\begin{table}[t]
\centering
\small
\begin{tabular}{lcccc}
\toprule
Unitization & 256 & 512 & 1024 & 2048 \\
\midrule
Sentence & \textbf{0.247} & \textbf{0.264} & \textbf{0.282} & 0.273 \\
Hierarchical & 0.246 & 0.262 & 0.279 & 0.271 \\
Sliding & 0.238 & 0.255 & 0.278 & \textbf{0.281} \\
Cluster & 0.230 & 0.258 & 0.275 & 0.278 \\
\bottomrule
\end{tabular}
\caption{Routing heuristic results on MIMIC with LLM generation. Policy: Lead for $B \leq 512$, MMR for $512 < B \leq 1024$, RCD for $B > 1024$.}
\label{tab:routing}
\end{table}

The routing heuristic closely tracks the oracle upper bound across all three datasets, demonstrating that a simple budget-based policy can approximate per-instance method selection. 
Table~\ref{tab:routing} shows routing results on MIMIC with LLM generation. Using Sentence-Unit with routing, the system achieves a ROUGE-1 score of 0.282 at 1024 tokens, matching the performance of the best single-method configuration. 

The heuristic assigns Lead to low budgets where positional structure is most informative, MMR to moderate budgets where redundancy filtering provides the largest gains, and RCD to high budgets where coverage becomes the limiting factor.

This routing policy provides a practical guideline: at tight budgets, prioritize non-redundant selection; at generous budgets, ensure broad coverage. While the current thresholds were tuned on extractive evaluation, the framework extends naturally to generative pipelines by recalibrating on LLM outputs.

\subsection{Generalization to PubMedQA}

To test whether our framework generalizes beyond summarization and beyond front-loaded documents, we evaluate on PQA-L, the long-context split of PubMedQA \citep{jin_pubmedqa_2019}. We run 500 instances at budgets of 64--256 tokens, using GPT-4o as the downstream classifier. Table~\ref{tab:pubmedqa} reports macro-F1.

\begin{table}[t]
\centering
\small
\begin{tabular}{lcccc}
\toprule
& 64 & 96 & 128 & 256 \\
\midrule
Lead & 0.233 & 0.284 & 0.372 & 0.587 \\
Shuffled & 0.412 & 0.484 & 0.514 & 0.587 \\
MMR & 0.328 & 0.409 & 0.467 & 0.598 \\
RCD & 0.396 & 0.457 & 0.503 & 0.598 \\
\midrule
Full context & \multicolumn{4}{c}{0.611} \\
\bottomrule
\end{tabular}
\caption{Macro-F1 on PQA-L (PubMedQA long-context split, $N{=}500$) with GPT-4o. Lead is worst at every budget; positional order is harmful for this task.}
\label{tab:pubmedqa}
\end{table}

Lead is the worst-performing selector at every budget. Shuffled selection outperforms Lead by up to 20 points, confirming that document-order position is actively harmful when key information appears later in the document. RCD provides the largest gains among principled selectors at low budgets, with the coverage and diversity terms compensating for the absence of a useful positional signal. At 256 tokens, MMR and RCD closely approach the full-context baseline (0.611), suggesting that principled selection substantially narrows the gap even in domains where front-loading does not hold.

\section{Discussion}
\label{sec:discussion}

The central finding of this work is that the optimal selection strategy depends on the evaluation paradigm. In the extraction evaluation (Table \ref{tab:r1_main}), Lead dominates at low budgets on MIMIC (0.242 vs 0.141 for MMR at 256 tokens),  because discharge notes front-load relevant content and extractive metrics reward direct lexical overlap. With LLM generation (Table \ref{tab:mimic_llm_rouge}), MMR outperforms other selectors at low budgets (0.267 vs 0.246 for Lead at 256 tokens), because the LLM benefits from non-redundant input even when generating abstractive summaries. This reversal suggests different bottlenecks: extractive metrics reward lexical overlap, which positional heuristics achieve by selecting content that shares vocabulary with early-appearing reference material. LLM generation is less sensitive to exact word choice since the model can paraphrase and reorganize. In this setting, redundant input wastes context-window capacity without adding information, thereby making diversity-aware selection more valuable.

The submodular structure of RCD provides approximation guarantees under knapsack constraints, but our experiments suggest that simpler objectives often suffice. MMR, which lacks explicit coverage guarantees, performs comparably to RCD across most settings. We conjecture that coverage is important when documents contain multiple disjoint themes with low cross-similarity -- a setting where MMR's local redundancy penalty cannot ensure global coverage. The sensitivity analysis (Appendix \ref{app:sensitivity}) shows that RCD performance is robust to weight variation, but the gains over MMR remain modest.


The routing heuristic achieves near-oracle performance on extractive evaluation (Figure \ref{fig:routing_bounds}), validating the intuition that the different objectives suit different budget regimes. However, the thresholds were tuned for extractive metrics. Our LLM results suggest recalibration: the front-loading index $\phi$ correctly identifies when positional selection suffices for extractive tasks, but may overestimate its value for generative pipelines where the model can identify salient content regardless of position.

The 1.7-point gap between Lead and Shuffled at 256 tokens is smaller than expected, given editorial conventions in clinical writing. Physicians often place diagnoses and chief complaints early, yet this positional structure provides only modest benefits when our LLM generates the final summary. One interpretation is that LLMs are robust to input ordering: given relevant content anywhere in the context, the model can identify and synthesize it effectively. This robustness is valuable for deployment but complicates the design of selection policies that rely on positional heuristics.

The PubMedQA results confirm that our framework does not depend on front-loading assumptions. When key evidence appears in later sections, Lead is the worst selector at every budget, and RCD's coverage and diversity terms compensate for the absence of useful positional signals. This validates the routing premise: the router should detect when positional selection is inappropriate and prefer diversity-aware methods.



BERTScore proves more informative than ROUGE for our task. The 2--4-point gaps in BERTScore between MMR and Lead are proportionally larger than the 1--2-point ROUGE gaps and align better with intuitions about selection quality. For evaluating LLM-based summarization pipelines, semantic similarity metrics may be more appropriate than lexical overlap.

The failure of Cluster-Unit is instructive.
We hypothesized that pre-grouping related sentences would help selectors allocate budget across topics.
Instead, clustering appears to create units that are either too large for fine-grained selection or that fragment the temporal structure of clinical narratives.
The admission-to-discharge arc of a hospital course may not align with semantic clusters.

The equivalence between Section-Unit and Sentence-Unit reveals a gap in implementation.
Parsing the section structure provides metadata that could inform scoring, but our current selectors ignore it.
Exploiting section structure for scoring remains an avenue for future work.

We stress-tested the RCD objective against weight perturbations on the simplex for each dataset. Appendix Figure~\ref{fig:sensitivity} visualizes the performance landscape at fixed budgets, and Figures~\ref{fig:mimic-stability}--\ref{fig:weight-sensitivity-suite} summarize tradeoffs and stability proxies. Two patterns are consistent across domains.
First, the landscape is smooth rather than brittle. Moving weight mass from relevance toward redundancy reduction gradually degrades ROUGE scores, indicating that gains are not driven by a knife-edge choice of $(\alpha,\beta,\gamma)$.
Second, stability increases with budget. The fraction of weight settings within a small tolerance of the optimum increases as $B$ grows, which supports the routing premise. At large budgets, multiple policies are effectively interchangeable, while at small budgets, the router matters. These observations also reduce the risk of overfitting. The router is calibrated on a development split, yet the held-out curves preserve the same ordering, and the near-optimal regions in weight space overlap across datasets.

\section{Conclusion}
\label{sec:conclusion}
We formalized budgeted context construction as a knapsack-constrained optimization problem and introduced RCD, a monotone submodular objective combining relevance, coverage, and diversity. Across MIMIC, Cochrane, and L-Eval, we find that the optimal selection strategy depends on the evaluation paradigm: Lead dominates for extractive evaluation at low budgets, while MMR provides consistent gains for LLM-based generation by filtering redundancy. Unitization matters less than selector choice, and ROUGE saturates for abstractive summaries while BERTScore better differentiates methods. A simple budget-based routing heuristic achieves near-oracle performance, providing practitioners with a practical guideline: prioritize diversity-aware selection at tight budgets when using LLMs, and ensure coverage at generous budgets. Experiments on PubMedQA confirm that our framework generalizes to classification tasks and to documents where positional heuristics are harmful. On MIMIC, budgeted selection at 512 tokens (roughly a quarter of median document length) matches full-context performance, demonstrating that principled selection can reduce cost without sacrificing quality. Future work includes section-aware relevance models for clinical note structure, tighter integration with learned compression operators, and robust calibration of routing policies under distribution shift.

\section{Limitations}

\label{sec:limitations}

This study isolates context construction as a separate module, clarifying the budgeted selection problem but omitting end-to-end training of the generator. As a result, the reported improvements reflect what can be achieved solely through selection with a fixed downstream model.
Our experiments focus on summarization-style tasks with reference summaries and one classification task.
The conclusions may not transfer directly to settings such as long-context question answering or tool use, where the notion of relevance is conditioned on a question and where correctness can depend on a single rare fact.
Similarly, our routing policy is tuned on held-out validation data within each dataset.
While we include sensitivity analyses over objective weights, the router could degrade under a distribution shift in note structure or writing style. 

All generation experiments use GPT-4o as the downstream model. Prior work has shown that zero-shot open-source LLMs underperform proprietary models on medical evidence summarization \citep{zhang_benchmarking_2024}, motivating our choice of a strong fixed generator to isolate selector effects. Evaluating with open-source LLMs would improve reproducibility and test whether the observed patterns are generator-dependent. Evaluation on substantially longer inputs, such as BigPatent \citep{sharma_bigpatent_2019}, ArXiv, or PubMed summarization \citep{cohan_discourse_2018}, would provide a more complete picture of scalability and generalizability.

Finally, evaluation relies on automatic metrics. ROUGE and BERTScore provide complementary signals, but neither directly measures factual correctness, clinical safety, or decision impact. Human evaluation and task-based clinical studies are needed to assess whether improvements in overlap and semantic similarity translate into safer and more useful summaries.

\section{Ethical Considerations}
Budgeting tokens in clinical NLP is not only an economic decision. A budgeted context is a lossy view of the record, and the loss is structured by the selection policy. This raises three obligations.

Any sentence-level selector should return provenance identifiers. When an LLM consumes $C_B$, users should be able to trace each generated claim to the originating sentence and inspect what was excluded under the current budget. In our pipeline, provenance is available by construction.

A fixed budget can under-serve notes that are long, atypical, or clinically complex. In practice, institutions should monitor performance as a function of document length, service, and patient acuity, and escalate budgets for high risk encounters. The router in Section~\ref{sec:routing} is compatible with such policies because it can be constrained to select conservative policies at small budgets and to fall back to Lead when uncertainty is high.

If budgets are imposed for cost reasons, the resulting quality degradation must not be systematically borne by particular patient groups or clinical services. We recommend routine stratified reporting and periodic recalibration of routing parameters on representative hospital data.

Finally, our experiments evaluate informativeness using overlap metrics. Overlap is not a sufficient safety criterion for clinical summarization. A deployment should include fact-checking, abstention in low-confidence cases, and human review for consequential decisions.

\section*{Acknowledgements}
This work was supported by the National Library of Medicine [grant numbers R01LM014344, R01LM014573] and the National Science Foundation (NSF) [grant numbers 2145640, 2139899].


\newpage
\appendix

\section{Extended results}

\renewcommand{\thetable}{A\arabic{table}}
\setcounter{table}{0}
\renewcommand{\thefigure}{A\arabic{figure}}
\setcounter{figure}{0}


\begin{table}[h!]
\centering
\small
\resizebox{\columnwidth}{!}{%
\begin{tabular}{llcccc}
\toprule
& & 256 & 512 & 1024 & 2048 \\
\midrule
\multicolumn{6}{l}{MIMIC} \\
& Lead & \textbf{0.242} & \textbf{0.246} & 0.249 & 0.230 \\
& Shuffled & 0.178 & 0.226 & 0.247 & 0.234 \\
& Sliding & 0.111 & 0.158 & 0.205 & 0.219 \\
& Hierarchical & 0.125 & 0.185 & 0.235 & 0.233 \\
& MMR & 0.141 & 0.209 & \textbf{0.262} & \textbf{0.249} \\
& RCD & 0.138 & 0.201 & 0.252 & 0.247 \\
& GraphCluster & 0.117 & 0.174 & 0.210 & 0.220 \\
\midrule
\multicolumn{6}{l}{Cochrane} \\
& Lead & 0.312 & 0.273 & 0.249 & \textbf{0.249} \\
& Shuffled & 0.306 & 0.272 & 0.249 & \textbf{0.249} \\
& Sliding & 0.298 & 0.267 & 0.249 & \textbf{0.249} \\
& Hierarchical & 0.309 & 0.273 & 0.249 & \textbf{0.249} \\
& MMR & \textbf{0.316} & 0.276 & \textbf{0.250} & \textbf{0.249} \\
& RCD & \textbf{0.316} & \textbf{0.278} & \textbf{0.250} & \textbf{0.249} \\
& GraphCluster & 0.284 & 0.261 & 0.249 & \textbf{0.249} \\
\midrule
\multicolumn{6}{l}{L-Eval} \\
& Lead & 0.207 & 0.177 & 0.126 & 0.079 \\
& Shuffled & 0.220 & 0.183 & 0.128 & 0.080 \\
& Sliding & \textbf{0.226} & \textbf{0.190} & \textbf{0.132} & \textbf{0.087} \\
& Hierarchical & 0.207 & 0.184 & 0.130 & 0.081 \\
& MMR & 0.220 & 0.182 & 0.128 & 0.081 \\
& RCD & 0.206 & 0.184 & 0.131 & 0.081 \\
& GraphCluster & 0.217 & 0.176 & 0.122 & 0.077 \\
\bottomrule
\end{tabular}}
\caption{Mean ROUGE-1 F1 across sentence-level selection policies.}
\label{tab:r1_main}
\end{table}

\begin{table}[h!]
\centering
\small
\begin{tabular}{llcccc}
\toprule
Unitization & Selector & 256 & 512 & 1024 & 2048 \\
\midrule
\multirow{4}{*}{Sentence}
  & Lead & .247 & .264 & .269 & .280 \\
  & Shuffled & .230 & .257 & .268 & .271 \\
  & MMR & .265 & .277 & \textbf{\textbf{.282}} & .279 \\
  & Facility & .247 & .267 & .274 & .273 \\
\midrule
\multirow{3}{*}{Hierarchical}
  & Lead & .246 & .262 & .272 & \textbf{.281} \\
  & MMR & \textbf{\textbf{.267}} & \textbf{\textbf{.279}} & .279 & .276 \\
  & Facility & .249 & .267 & .271 & .271 \\
\midrule
\multirow{3}{*}{Sliding}
  & Lead & .238 & .255 & .263 & .274 \\
  & MMR & .254 & .276 & .278 & .277 \\
  & Facility & .237 & .269 & .276 & \textbf{\textbf{.281}} \\
\midrule
\multirow{3}{*}{Cluster}
  & Lead & .230 & .258 & .273 & .277 \\
  & MMR & .258 & .277 & .275 & .280 \\
  & Facility & .240 & .271 & .275 & .278 \\
\bottomrule
\end{tabular}
\caption{ROUGE-1 F1 on MIMIC with LLM generation.}
\label{tab:mimic_llm_rouge}
\end{table}

\begin{table}[t]
\centering
\small
\begin{tabular}{llcccc}
\toprule
Unitization & Selector & 256 & 512 & 1024 & 2048 \\
\midrule
\multirow{4}{*}{Sentence}
  & Lead & .092 & .101 & .112 & \textbf{.121} \\
  & Shuffled & .074 & .101 & .109 & .110 \\
  & MMR & \textbf{\textbf{.112}} & .117 & \textbf{\textbf{.120}} & \textbf{.119} \\
  & Facility & .097 & .110 & .116 & .113 \\
\midrule
\multirow{3}{*}{Hierarchical}
  & Lead & .091 & .099 & .113 & .118 \\
  & MMR & .111 & \textbf{\textbf{.118}} & .119 & .115 \\
  & Facility & .098 & .112 & .116 & .115 \\
\midrule
\multirow{3}{*}{Sliding}
  & Lead & .086 & .094 & .101 & .115 \\
  & MMR & .104 & .114 & .115 & .117 \\
  & Facility & .089 & .110 & .112 & .115 \\
\midrule
\multirow{3}{*}{Cluster}
  & Lead & .083 & .098 & .112 & .116 \\
  & MMR & .104 & \textbf{.118} & .115 & .116 \\
  & Facility & .096 & .110 & .115 & .117 \\
\bottomrule
\end{tabular}
\caption{BERTScore F1 on MIMIC with LLM generation. Scores are rescaled with baseline correction. Values represent semantic similarity above a random baseline.}
\label{tab:mimic_llm_bert}
\end{table}

\newpage
\section{Theoretical Properties of RCD}
\label{app:theory}

We recall standard definitions.
A set function $f:2^{[n]}\to \R$ is \textbf{submodular} if for all $A\subseteq B\subseteq[n]$ and all $e\notin B$,
\begin{equation}
\Delta_f(e\mid A) \ge \Delta_f(e\mid B),
\end{equation}
where $\Delta_f(e\mid S) = f(S\cup\{e\})-f(S)$ denotes the marginal gain.
A function is \textbf{monotone} if $f(A)\le f(B)$ whenever $A\subseteq B$.

\paragraph{Facility location coverage.}
Define $C(S)=\sum_{i=1}^n \max_{j\in S} k(i,j)$, with the convention that $\max_{j\in \emptyset} k(i,j)=0$.
We assume $k(i,j)\ge 0$.

\begin{lemma}
$C$ is monotone submodular.
\end{lemma}

\begin{proof}
Fix $i$ and define $c_i(S)=\max_{j\in S} k(i,j)$.
Monotonicity is immediate since adding elements can only increase the maximum.
For submodularity, let $A\subseteq B$ and $e\notin B$.
Write $m_A=c_i(A)$ and $m_B=c_i(B)$.
Since $A\subseteq B$, we have $m_A\le m_B$.
The marginal gains satisfy
\begin{align*}
\Delta_{c_i}(e\mid A) &= \max\{m_A,k(i,e)\}-m_A,\\
\Delta_{c_i}(e\mid B) &= \max\{m_B,k(i,e)\}-m_B.
\end{align*}
If $k(i,e)\le m_B$ then $\Delta_{c_i}(e\mid B)=0$ and $\Delta_{c_i}(e\mid A)\ge 0$.
If $k(i,e)> m_B$ then $\Delta_{c_i}(e\mid B)=k(i,e)-m_B$ and $\Delta_{c_i}(e\mid A)=k(i,e)-m_A\ge k(i,e)-m_B$.
In both cases $\Delta_{c_i}(e\mid A)\ge \Delta_{c_i}(e\mid B)$, so $c_i$ is submodular.
Finally, $C$ is a nonnegative sum of submodular functions, so it is submodular \citep{krause_golovin_2014}.
\end{proof}

\paragraph{Log-determinant diversity.}
Assume $K\succeq 0$ and $\eta>0$, and define $D(S)=\log\det(I+\eta K_S)$.

\begin{lemma}
$D$ is monotone submodular.
\end{lemma}

\begin{proof}
Since $K\succeq 0$, there exists a feature map $\Phi\in \R^{n\times d}$ such that $K=\Phi\Phi^\top$.
Let $\Phi_S$ denote the submatrix with rows indexed by $S$.
By Sylvester's determinant identity,
\begin{align*}
\det(I+\eta K_S) &= \det(I+\eta \Phi_S\Phi_S^\top)\\
&= \det(I+\eta \Phi_S^\top \Phi_S).
\end{align*}
Let $M_S = I+\eta \Phi_S^\top \Phi_S \succ 0$.
For $e\notin S$, the matrix determinant lemma gives
\[
\det(M_{S\cup\{e\}})=\det(M_S)\left(1+\eta\,\phi_e^\top M_S^{-1}\phi_e\right),
\]
where $\phi_e$ is the feature vector for item $e$.
Taking logs yields the marginal gain
\begin{equation}
\Delta_D(e\mid S)=\log\!\left(1+\eta\,\phi_e^\top M_S^{-1}\phi_e\right).
\label{eq:logdet_marginal}
\end{equation}
Monotonicity holds because $M_S^{-1}\succeq 0$ implies the argument of the log is at least $1$.
For submodularity, note that if $A\subseteq B$ then $M_A \preceq M_B$ and hence $M_A^{-1}\succeq M_B^{-1}$ in the Loewner order.
Therefore $\phi_e^\top M_A^{-1}\phi_e \ge \phi_e^\top M_B^{-1}\phi_e$, and \eqref{eq:logdet_marginal} implies $\Delta_D(e\mid A)\ge \Delta_D(e\mid B)$ since $\log(1+x)$ is increasing.
Thus $D$ is submodular \citep{kulesza_taskar_2012,krause_golovin_2014}.
\end{proof}

\paragraph{RCD objective.}
Let $R(S)=\sum_{i\in S} r_i$ with $r_i\ge 0$.

\begin{proposition}
If $\alpha,\beta,\gamma\ge 0$, $k(i,j)\ge 0$, and $K\succeq 0$, then
\[
F(S)=\alpha R(S)+\beta C(S)+\gamma D(S)
\]
is monotone submodular.
\end{proposition}

\begin{proof}
$R$ is modular and hence submodular.
By the previous lemmas, $C$ and $D$ are monotone submodular.
A nonnegative linear combination of submodular functions is submodular, and a nonnegative linear combination of monotone functions is monotone \citep{krause_golovin_2014}.
\end{proof}

\paragraph{Optimization under a token knapsack.}
Maximizing a monotone submodular $F$ under a knapsack constraint $\sum_{i\in S}c_i\le B$ is NP-hard in general.
There exist polynomial-time algorithms achieving a $(1-1/e)$ approximation \citep{sviridenko_2004}.
In practice, we use a lazy greedy variant that iteratively selects the item with the largest marginal gain per token cost, combined with a best-singleton check.
This heuristic is widely used in large-scale summarization and often performs close to the theoretical algorithms while remaining simple to implement.

\section{Ensuring Positive Semidefiniteness}
\label{app:psd}
When $K=XX^\top$ for an embedding matrix $X$, positive semidefiniteness holds by construction.
Otherwise, we project by symmetrization and eigenvalue clipping: compute $K \leftarrow \frac{1}{2}(K+K^\top)$, perform eigendecomposition, replace negative eigenvalues with zero, and add small diagonal jitter.
\section{Performance Metrics Definitions}
\label{app:metrics}

\subsection{ROUGE-1}
Let $y$ denote the generated text and $y^*$ denote the reference summary. Then, let $U(y)$ denote the multiset of unigrams in $y$. Precision and recall are defined as:
\begin{align}
P_1 &= \frac{|U(y) \cap U(y^*)|}{|U(y)|}, &
R_1 &= \frac{|U(y) \cap U(y^*)|}{|U(y^*)|}.
\end{align}
And ROUGE-1 is the harmonic mean:
\begin{equation}
\text{ROUGE-1} = \frac{2 P_1 R_1}{P_1 + R_1}
\end{equation}

\subsection{ROUGE-2}
Let $B(y)$ denote the multiset of consecutive word pairs in $y$. Precision, recall and F1 are defined analogously:
\begin{equation}
P_2 = \frac{|B(y) \cap B(y^*)|}{|B(y)|}, \quad
R_2 = \frac{|B(y) \cap B(y^*)|}{|B(y^*)|}.
\end{equation}
\begin{equation}
\text{ROUGE-2} = \frac{2 P_2 R_2}{P_2 + R_2}, \quad
\end{equation}

\subsection{BERTScore}
Let $\mathbf{y} = (\mathbf{y}_1, \ldots, \mathbf{y}_m)$ and $\mathbf{y}^* = (\mathbf{y}^*_1, \ldots, \mathbf{y}^*_n)$ denote the contextualized token embeddings from a pretrained model (we use DeBERTa-xlarge-MNLI).
Precision and recall are computed via greedy maximum cosine similarity matching:
\begin{align}
P_{\text{BERT}} &= \frac{1}{m} \sum_{i=1}^{m} \max_{j} \, \mathbf{y}_i^\top \mathbf{y}^*_j, \\
R_{\text{BERT}} &= \frac{1}{n} \sum_{j=1}^{n} \max_{i} \, \mathbf{y}_i^\top \mathbf{y}^*_j.
\end{align}
BERTScore F1 is the harmonic mean:
\begin{equation}
\text{BERTScore} = \frac{2 P_{\text{BERT}} R_{\text{BERT}}}{P_{\text{BERT}} + R_{\text{BERT}}}.
\end{equation}

\section{Sensitivity Analysis}
\label{app:sensitivity}

Figures~A1–A4 collectively show that the RCD objective is stable across a broad range of weight configurations and budgets, rather than being tuned to a narrow parameter regime. Figure~\ref{fig:sensitivity} demonstrates that, for a fixed budget, the mean ROUGE-1 score varies only modestly as $(\alpha,\beta,\gamma)$ are perturbed, with no abrupt degradation when shifting mass between relevance, coverage, and diversity. Figure~\ref{fig:mimic-tradeoff}  makes this observation quantitative by plotting mean ROUGE-1 against the Wasserstein distance from the best-performing weight vector: performance decreases smoothly and approximately monotonically as the distance increases, indicating that suboptimal weights induce gradual, not catastrophic, loss. Figure~\ref{fig:mimic-stability}, ~\ref{fig:cochrane-stability} further shows that the fraction of weight settings within a small tolerance of the best score grows rapidly with the budget, approaching one for larger budgets on both MIMIC and Cochrane, which implies that RCD becomes increasingly insensitive to precise weight choice as more tokens are available. Finally, the simplex visualizations in Figure~\ref{fig:mimic-simplex}, ~\ref{fig:cochrane-simplex} reveal wide plateaus of near-optimal solutions rather than isolated optima, with the empirically selected weights (highlighted in yellow) lying in the interior of these regions. Taken together, these results provide strong evidence that RCD is not overfitting to a specific $(\alpha,\beta,\gamma)$ configuration and that its empirical gains are driven by the structure of the objective itself, not by fragile hyperparameter tuning.

\begin{figure*}[!h]
\centering
\begin{subfigure}{0.32\linewidth}
\centering
\includegraphics[width=\linewidth]{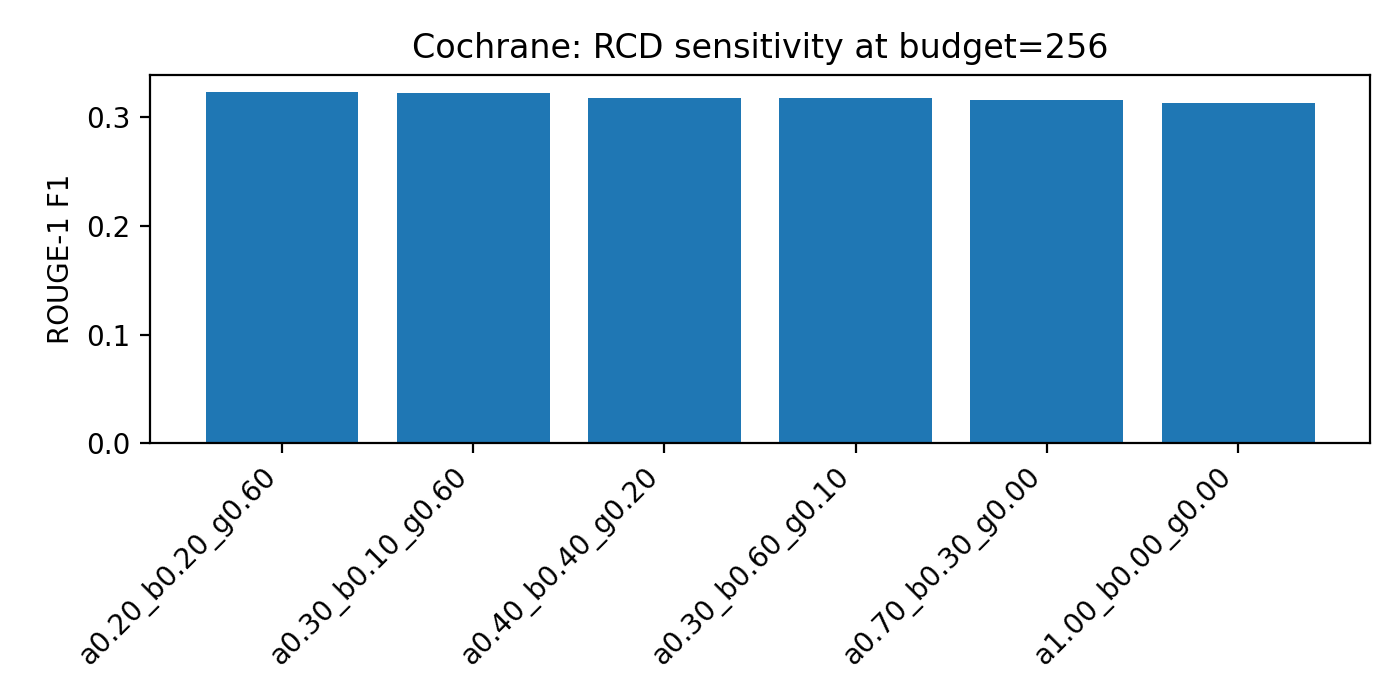}
\caption{Cochrane}
\end{subfigure}
\begin{subfigure}{0.32\linewidth}
\centering
\includegraphics[width=\linewidth]{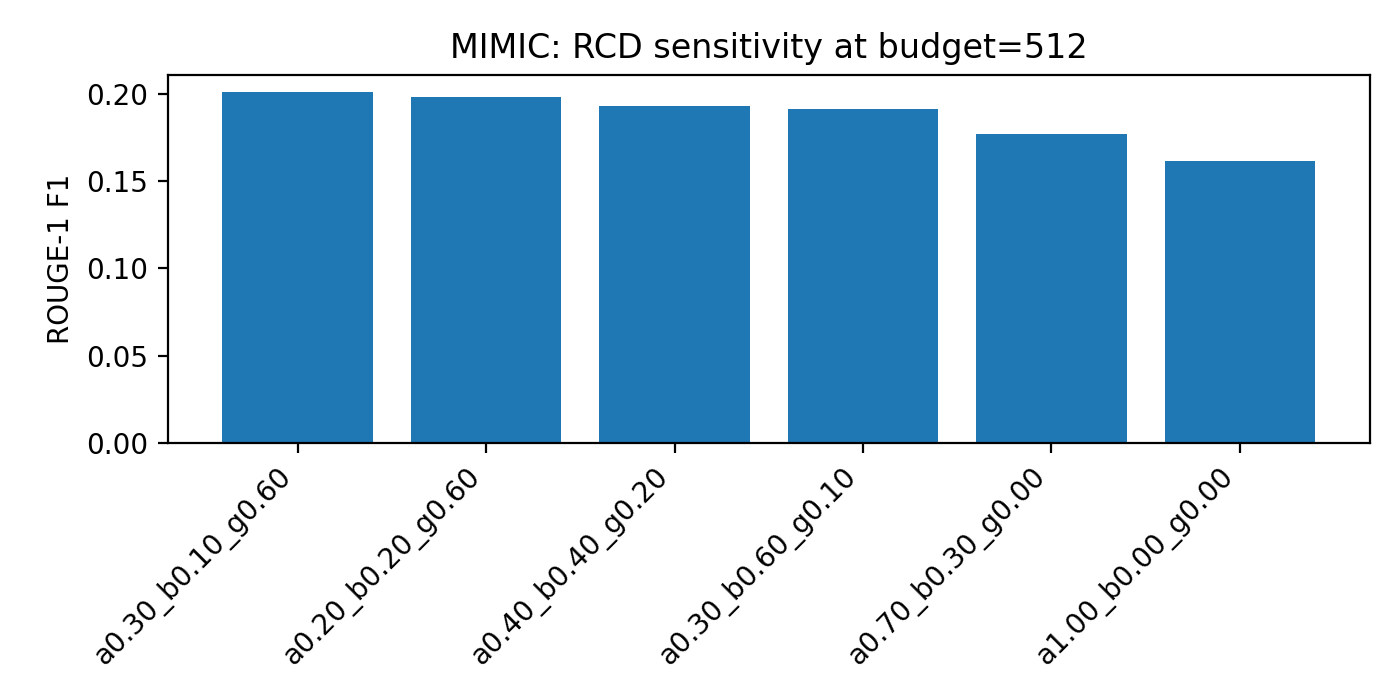}
\caption{MIMIC}
\end{subfigure}
\begin{subfigure}{0.32\linewidth}
\centering
\includegraphics[width=\linewidth]{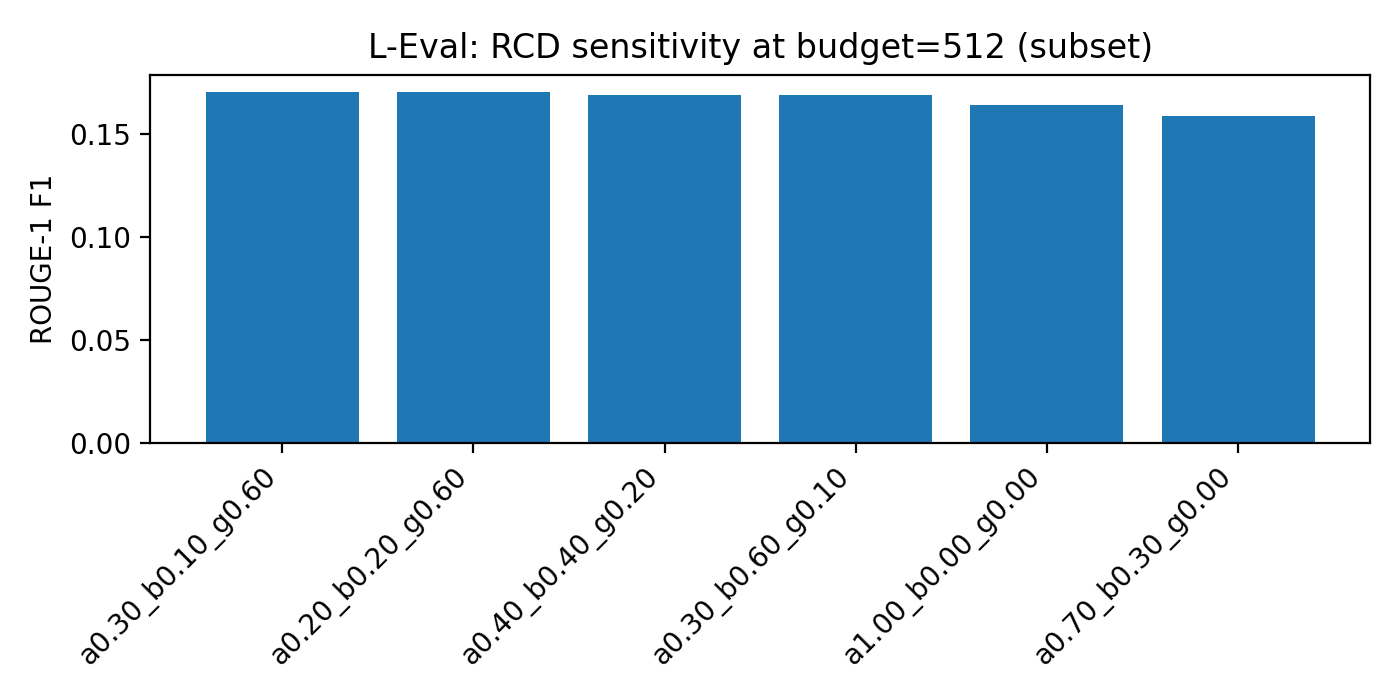}
\caption{L-Eval}
\end{subfigure}
\caption{Sensitivity of RCD to the objective weights $\alpha,\beta,\gamma$.}
\label{fig:sensitivity}
\end{figure*}

\begin{figure*}[!h]
    \centering
    \includegraphics[width=.67\linewidth]{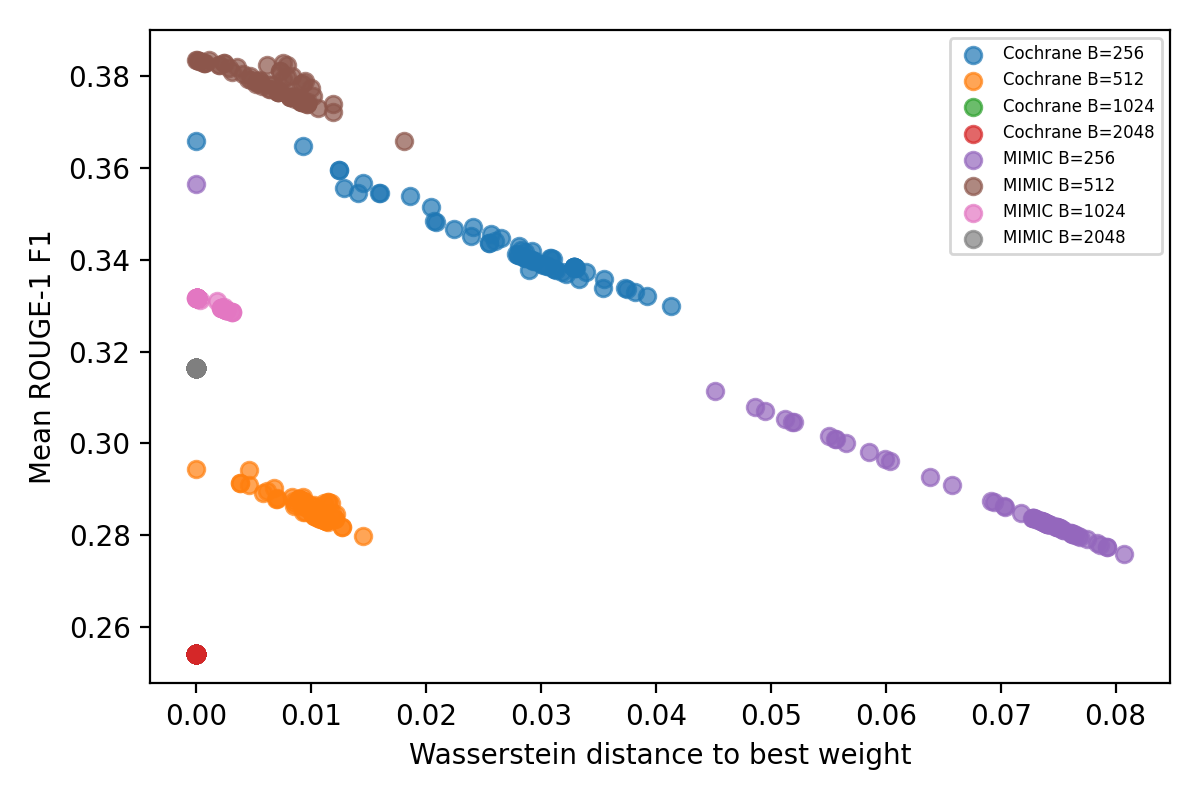}
    \caption{Mean ROUGE-1 F1 Score vs. Wasserstein Distance from Best Weight by dataset and budget}
    \label{fig:mimic-tradeoff}
\end{figure*}

\begin{figure*}[!h]
\centering
  \begin{subfigure}[b]{0.4\textwidth}
    \centering
    \includegraphics[width=\linewidth,clip,trim=0 0 0 1.5em]{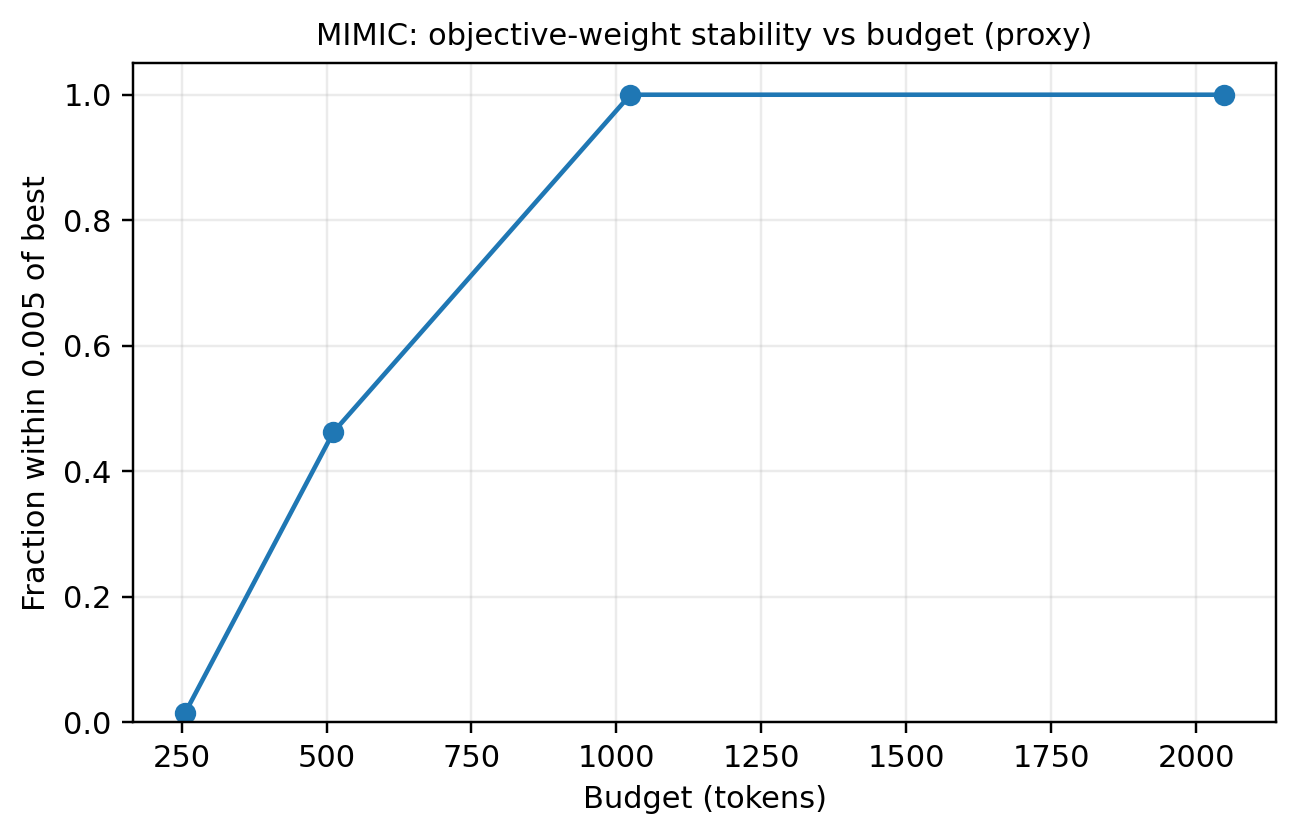}
    \caption{MIMIC}
    \label{fig:mimic-stability}
  \end{subfigure}
  \hspace{3em}
  \begin{subfigure}[b]{0.4\textwidth}
    \centering
    \includegraphics[width=\linewidth,clip,trim=0 0 0 1.5em]{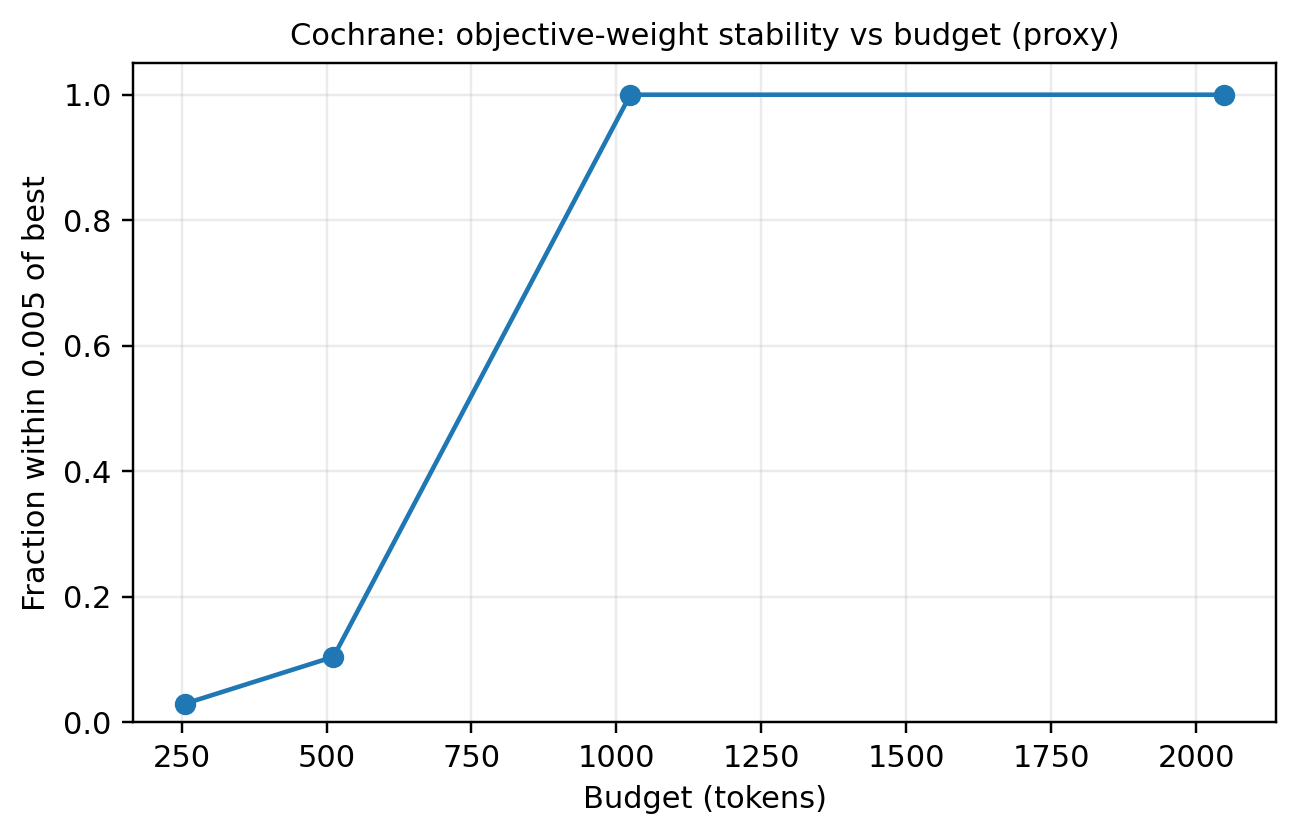}
    \caption{Cochrane}
    \label{fig:cochrane-stability}
  \end{subfigure}
\caption{Robustness proxy plots for MIMIC and Cochrane across budgets and objectives.}
\end{figure*}

\begin{figure*}[!h]
\centering
  \begin{subfigure}[b]{0.44\textwidth}
    \centering
    \includegraphics[width=\linewidth]{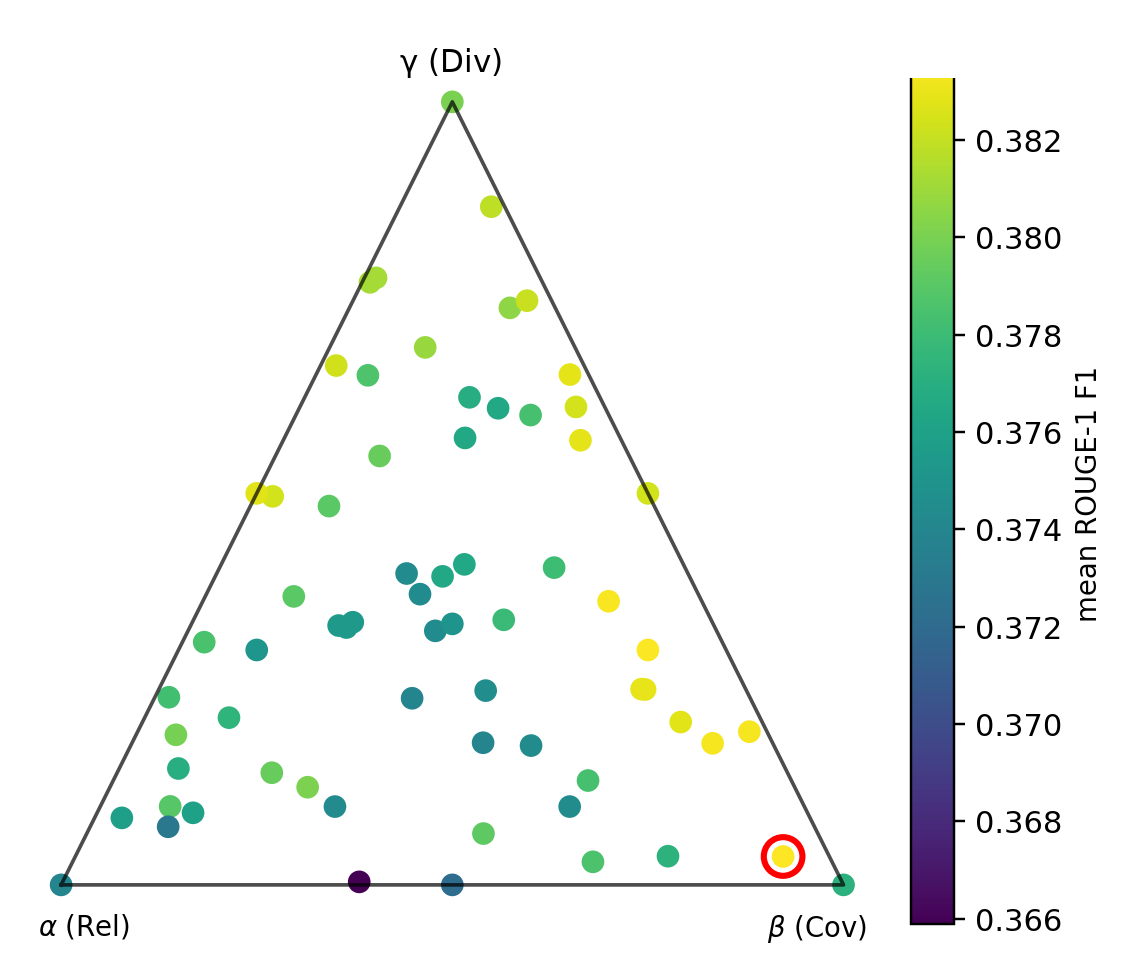}
    \caption{MIMIC}
    \label{fig:mimic-simplex}
  \end{subfigure}
  \hspace{3em}
  \begin{subfigure}[b]{0.44\textwidth}
    \centering
    \includegraphics[width=\linewidth]{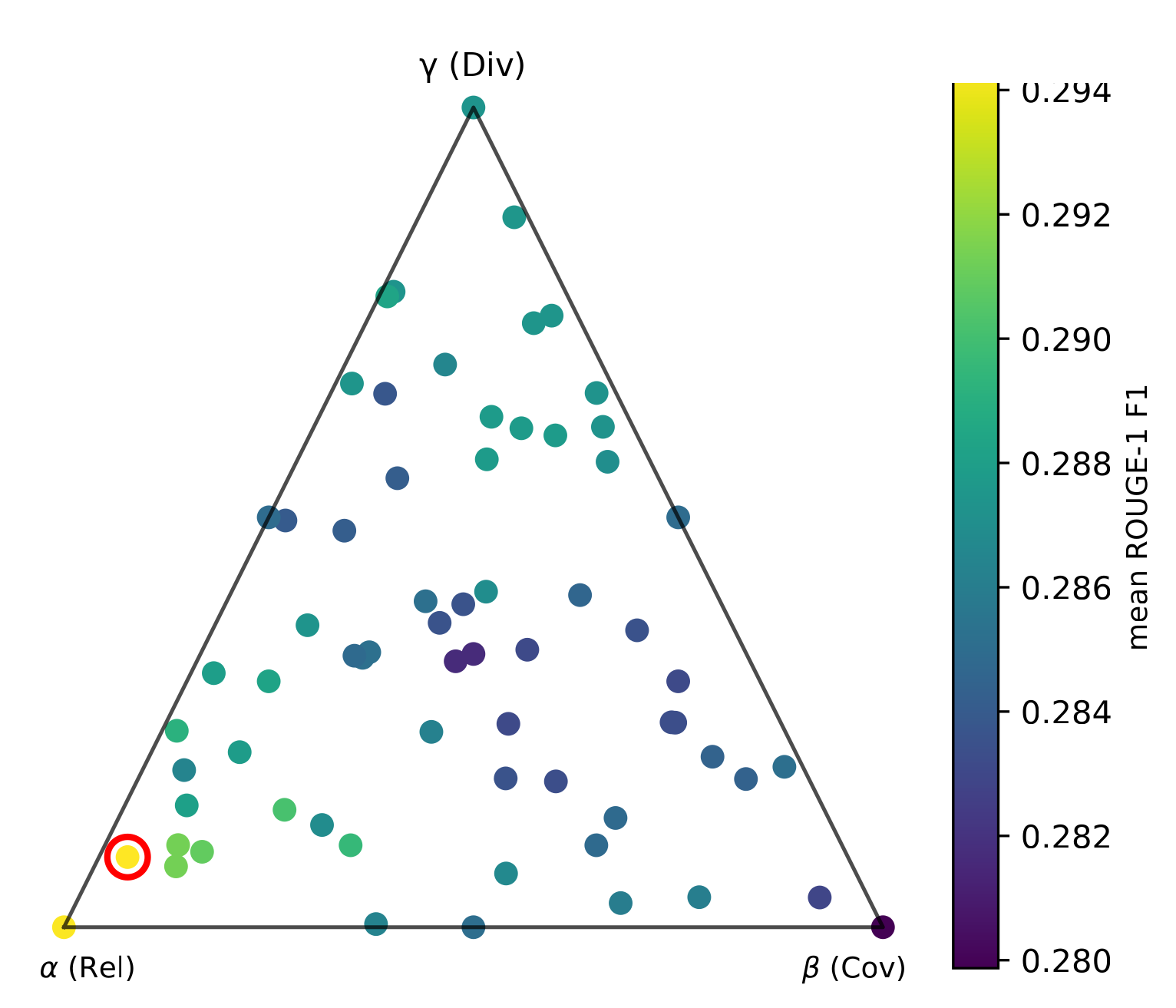}
    \caption{Cochrane}
    \label{fig:cochrane-simplex}
  \end{subfigure}
  \caption{Weight sensitivity on MIMIC and Cochrane ($B=512$).}
  \label{fig:weight-sensitivity-suite}
\end{figure*}

\end{document}